\documentclass{article}

\usepackage{hyperref}

\usepackage[olditem,oldenum]{paralist}
\usepackage{enumitem}
\usepackage{amsmath}
\usepackage{amsthm}
\usepackage{amssymb}
\usepackage{subfig}
\usepackage{color}
\usepackage[english]{babel}
\usepackage{graphicx}
\usepackage{grffile}
\usepackage{wrapfig,epsfig}
\usepackage{epstopdf}
\usepackage{url}
\usepackage{color}
\usepackage{epstopdf}
\usepackage{bbm}
\usepackage{comment}
\usepackage{dsfont}

\usepackage[utf8]{inputenc} 
\usepackage[T1]{fontenc}    
\usepackage{booktabs}       
\usepackage{amsfonts}       
\usepackage{nicefrac}       
\usepackage{microtype}      
\usepackage{multicol}
\usepackage[dvipsnames]{xcolor}

\hypersetup{colorlinks=true,citecolor=blue,linkcolor=blue}

\usepackage{scalerel,amssymb} 
 
\def\msquare{\mathord{\scalerel*{\Box}{gX}}} 
\newtheorem{theorem}{Theorem}[section]
\newtheorem{lemma}[theorem]{Lemma}
\newtheorem{definition}[theorem]{Definition}

\newtheorem{assumption}[theorem]{Assumption}

\newcommand{\RR}{\mathbb{R}}

\newcommand{\wh}{\widehat}
\newcommand{\wt}{\widetilde}

\newcommand{\calS}{\mathcal{S}}
\newcommand{\calA}{\mathcal{A}}
\newcommand{\calB}{\mathcal{B}}
\newcommand{\calF}{\mathcal{F}}
\newcommand{\R}{\mathbb{R}}
\renewcommand{\P}{\mathbb{P}}

\renewcommand{\varepsilon}{\epsilon}
\renewcommand{\tilde}{\wt}
\renewcommand{\hat}{\wh}

\DeclareMathOperator*{\E}{{\mathbb{E}}}

\DeclareMathOperator{\hQ}{{\hat{Q}}}
\DeclareMathOperator{\hV}{{\hat{V}}}
\DeclareMathOperator{\rad}{{\texttt{rad}}}
\DeclareMathOperator{\rel}{{\texttt{rel}}}
\DeclareMathOperator{\ind}{{\texttt{index}}}
\DeclareMathOperator{\dist}{{\texttt{dist}}}
\DeclareMathOperator{\dom}{{\texttt{dom}}}
\DeclareMathOperator{\pa}{{\texttt{pa}}}
\DeclareMathOperator{\init}{{\texttt{init}}}

\usepackage[accepted]{icml2019}

\icmltitlerunning{\textsc{ZoomRL}}

\begin{document}

\twocolumn[
\icmltitle{Zooming for Efficient Model-Free Reinforcement Learning in Metric Spaces}




\begin{icmlauthorlist}
\icmlauthor{Ahmed Touati}{mila}
\icmlauthor{Adrien Ali Taiga}{mila,goo}
\icmlauthor{Marc G. Bellemare}{goo,cifar}
\end{icmlauthorlist}

\icmlaffiliation{mila}{Mila, Universitt\'e de Montr\'eal}
\icmlaffiliation{goo}{Google Research, Brain team}
\icmlaffiliation{cifar}{CIFAR Fellow}
\icmlcorrespondingauthor{Ahmed Touati}{ahmed.touati@umontreal.ca}



\icmlkeywords{Machine Learning, ICML}

\vskip 0.3in
]



\printAffiliationsAndNotice{}  

\begin{abstract}
Despite the wealth of research into provably efficient reinforcement learning algorithms, most works focus on tabular representation and thus struggle to handle exponentially or infinitely large state-action spaces. In this paper, we consider episodic reinforcement learning with a continuous state-action space which is assumed to be equipped with a natural metric that characterizes the proximity between different states and actions. We propose \textsc{ZoomRL}, an online algorithm that leverages ideas from continuous bandits to learn an adaptive discretization of the joint space by \textit{zooming} in more promising and frequently visited regions while carefully balancing the exploitation-exploration trade-off. We show that \textsc{ZoomRL} achieves a worst-case regret $\tilde{O}(H^{\frac{5}{2}} K^{\frac{d+1}{d+2}})$ where $H$ is the planning horizon, $K$ is the number of episodes and $d$ is the covering dimension of the space with respect to the metric. Moreover, our algorithm enjoys improved metric-dependent guarantees that reflect the geometry of the underlying space. Finally, we show that our algorithm is robust to small misspecification errors.
\end{abstract}

\section{Introduction}
Reinforcement learning~\citep{sutton1998introduction} (RL) is a framework for solving sequential decision-making problems. Through trial and error an agent must learn to act optimally in an unknown environment in order to maximize its expected utility. Efficient learning requires balancing exploration (acting to gain more knowledge) and exploitation (acting optimally according to the available knowledge).

\textit{Optimism in the face of uncertainty} (OFU) is one of the traditional guiding principles that offers provably efficient
learning algorithms. We can distinguish two classes of approaches: confidence-intervals based methods~\citep{kearns2002near, strehl2005theoretical, jaksch2010near} and exploration-bonus based methods~\citep{azar2017minimax, jin2018q, jian2019exploration}. In the former, the agent builds a set of statistically plausible Markov Decision Processes (MDPs) that contains the true MDP with high probability. Then, the agent selects the most optimistic version of its model and acts optimally with respect to it. In the latter, discoveries of poorly understood states and actions are rewarded by an exploration bonus. Such bonus is designed to bound estimation errors on the value function.

In the regime of MDPs with a finite state-action space, the OFU principle has been successfully implemented and efficient algorithms typically achieve regret that scales sublinearly with the number of discrete states and the number of discrete actions. This precludes applying them to arbitrarily large state-action spaces. On the other hand, MDPs with continuous state-action spaces have been an active area of investigation~\citep{ortner2012online, lakshmanan2015improved, song2019efficient}. A common theme is to assume some structure knowledge, such as the existence of similarity metric between state-action pairs, and then to use a uniform discretisation of the space or nearest-neighbor approximators.

In this work, we focus on the finite-horizon MDP formalism with an unknown transition kernel. We suppose that the state-action space is equipped by a metric that characterizes the proximity between different states and actions.  Such metrics have been studied in previous work for state aggregation~\citep{ferns2004metrics, ortner2007pseudometrics}. We assume that the optimal action-value function is Lipschitz continuous with respect to this metric, which means that state-action pairs that are close to each other have similar optimal values. 

We propose an online model-free RL algorithm, \textsc{ZoomRL}, that actively explores the state-action space by learning on-the-fly an adaptive partitioning. Algorithms based on uniform partitions, such as the works in~\citet{ortner2012online} and~\citet{song2019efficient}, disregard the shape of the optimal value function and thus could waste effort in partitioning irrelevant regions of the space. Moreover, the granularity of the partition should be tuned and it depends on the time horizon and the covering dimension of joint space. In contrast, \textsc{ZoomRL} is able to take advantage of the structure of the problem's instance at hand by adjusting the discretisation to frequently visited and high-rewarding regions to get better estimates. Zooming approaches have been successfully applied in Lipschitz bandits~\citep{kleinberg2008multi} and continuous contextual bandits~\citep{slivkins2014contextual}. However, in the bandit setting, an algorithm's cumulative regret can be easily decomposed into regret incurred in each sub-partition which is controlled by the size of the sub-partition itself. In contrast, in the reinforcement learning setting, the errors are propagated through iterations and we need to carefully control how they accumulate over iterations and navigate through sub-partitions. We show that \textsc{ZoomRL} achieves a worst-case regret $\tilde{O}(H^{\frac{5}{2}} K^{\frac{d+1}{d+2}})$ where $H$ is the planning horizon, $K$ is the number of episodes and $d$ is the covering dimension of the space with respect to the metric. Moreover, \textsc{ZoomRL} enjoys an improved metric-dependent guarantee that reflects the geometry of the underlying space and whose scaling in terms of $K$ is optimal as it matches the lower bound in continuous contextual bandit~\citep{slivkins2014contextual} when $H=1$. Finally, we study how our algorithm cope with the misspecified setting (Assumption~\ref{assum:approx_lipschitz}). We show that it is robust to small misspecification error as it suffers only from an additional regret term $O(H K \epsilon)$ if the true optimal action-value function is Lipschitz up to an additive error uniformly bounded in absolute value by $\epsilon$.

\section{Related Work}

\textbf{Exploration in metric spaces:}
There have been several recent works that study exploration in continuous state-action MDPs under different structured assumptions. \citet{kakade2003exploration} assume a local continuity of the reward function and the transition kernel with respect to a given metric. They propose a generalization of the $E^3$ algorithm of \citet{kearns2002near} to metric spaces.
Their sample complexity depends on the covering number of the space under the continuity metric instead of the number of the states. However, their algorithm requires access to an approximate planning oracle. \citet{lattimore2013sample} assume that the true transition kernel belongs to a finite or compact hypothesis class.
Their algorithm consists in maintaining a set of transitions models and pruning it over time by eliminating the provable implausible models.
They establish a sample complexity that depends polynomially on the cardinality or covering number of the model class. \citet{pazis2013pac} consider a continuous state-action MDP, develop a nearest-neighbor based algorithm under the assumption that all Q-functions encountered are Lipschitz continuous, showing a sample complexity that depends on an approximate covering number.
\citet{ortner2012online} develop a model-based algorithm that combines state aggregation with the standard UCRL2 algorithm~\citep{jaksch2010near} under the assumption of Lipschitz or H\"older continuity of rewards and transition kernel and they establish a regret bound scaling in $K^\frac{2 d + 1}{2 d + 2}$ where $d$ in the dimension of the state space and $K$ is the number of episodes.
\citet{lakshmanan2015improved} improve the latter work by considering a kernel density estimator instead of a frequency estimator for the transition probabilities. They achieve a regret bound of $K^\frac{d + 1}{d + 2}$. \citet{yang2019learning} consider a deterministic control system under a Lipschitz assumption of the optimal action-value functions and the transition function and they establish a regret of $K^\frac{d-1}{d}$ where $d$ here is the doubling dimension.
Recently,~\citet{song2019efficient} extended the tabular $Q$-learning with upper-confidence bound exploration strategy, developed in~\citet{jin2018q}, to continuous state-action MDPs using a uniform discretisation of the joint space leading to the regret bound $\tilde{O}(H^{\frac{5}{2}} K^{\frac{d+1}{d+2}})$ where $H$ is the planning horizon and $d$ is the covering dimension.
They only assume that the optimal action-value function is Lipschitz continuous.
This assumption is more general than that used in the aforementioned works as it is known that Lipschitz continuity of the reward function and the transition kernel leads to Lipschitz continuity of the optimal action-value function~\citep{asadi2018lipschitz}.
We use the same condition in this present paper.

\textbf{Adaptive discretization:} Our method is closely related to methods that learn partition from continuous bandit literature~\citep{kleinberg2008multi, bubeck2009online, slivkins2014contextual,azar2014online, munos2014bandits}. In particular, our method is inspired by the contextual Zooming algorithm introduced in~\citet{slivkins2014contextual} for contextual bandits, that we extend in non-trivial way to episodic RL setting. Our method is similar to two recently proposed algorithms.~\citet{zhu2019stochastic} propose and analyze an adaptive partitioning algorithm approach in the specific case where the metric space is a subset of $\RR^d$ equipped with $l_\infty$ distance as similarity metric. Concurrently to our work, ~\citet{sinclair2019adaptive} extend the latter result to any generic metric space. However, their algorithm \textsc{Adaptive Q-learning} requires, at each re-partition step, a packing oracle that is able to take a region and value $r$ and outputs an $r$-packing of that region. Whereas, our algorithm is oracle-free and creates at most a single sub-region when needed. More comparison with this work requires the introduction of some notations and is therefore deferred to Section~\ref{sec:the algorithm}.

\section{Problem Statement}

\subsection{Episodic Reinforcement Learning and Regret}
We consider a finite horizon MDP $({\cal S}, {\cal A}, \P, \mathbf{r}, H)$ where $\cal S$ and $\cal A$ are the state and action space, $H$ is the planning horizon i.e number of steps in each episode, $\P$ is the transition kernel such that $\P_h(\cdot | s,a)$ gives the distribution over next states if action $a$ is taken at state $s$ at step $h \in [H]$, $\mathbf{r}$ is the reward function such that $r_h(s,a)\in [0,1]$ is the reward of taking action $a$ at state $s$ at time step $h$. For any step $h \in [H]$ and $(s, a) \in {\cal S} \times {\cal A}$, the state-action value function of a non-stationary policy $\pi = (\pi_1, \ldots, \pi_H)$ is defined as 
$
Q^{\pi}_h(s, a) = r_h(s, a) + \E \left[ \sum_{i=h+1}^H r_i(s_i, \pi_i(s_i)) ~\Big|~ s_h = x, a_h = a\right],
$ and the value function is $V_h^{\pi}(s) = Q^\pi_h(s, \pi_h(s))$. As the horizon is finite, under some regularity conditions~\citep{shreve1978alternative}, there always exists an optimal policy $\pi^\star$ whose value and action-value functions are defined as $V^\star_h(x) \triangleq V^{\pi^\star}_h(s) = \max_{\pi}V^{\pi}_h(s)$ and $Q^\star_h(s, a) \triangleq Q^{\pi^\star}_h(s, a) = \max_{\pi}Q^{\pi}_h(s, a)$. If we denote $[\P_h V_{h+1}](s, a) = \E_{s' \sim \P_h(\cdot \mid s, a)}[V_{h+1}(s')]$, both $Q^\pi$ and $Q^\star$ can be conveniently written as the result of the following Bellman equations
\begin{align}
    Q^\pi_h(s, a) & = r_h(s, a) + [\P_h V^\pi_{h+1}](s, a) \label{eq: bellman_eq}, \\
    Q^\star_h(s, a) & = r_h(s, a) + [\P_h V^\star_{h+1}](s, a) \label{eq: bellman_opt},
\end{align}
where $V^\pi_{H+1}(s) = V^\star_{H+1}(s) = 0$ and $V^\star_h(s) = \max_{a \in {\cal A}}Q^\star_h(s, a)$, for all $s \in \cal S$.

We focus on the online episodic reinforcement learning setting in which the reward and the transition kernel are unknown. The learning agent plays the game for $K$ episodes $k=1, \ldots, K$, where each episode $k$ starts from some initial state $s^k_1$ sampled according to some initial distribution. The agent controls the system by choosing a policy $\pi_k$ at the beginning of the $k$-th episode. The total expected regret is defined then
\begin{align*}
    \textsc{Regret}(K) = \sum_{k=1}^K V^*_1(s_1^k) - V_1^{\pi_k}(s_1^k).
\end{align*}

\subsection{Metric space}
We assume that the state-action space $\cal X \triangleq {\cal S} \times \cal A$ is compact endowed with a metric $\texttt{dist}:\cal X \times \cal X \to \R^+$. This leads us to state our main assumption: 

\begin{assumption}[Lipschitz Continuous $Q^\star$]\label{assum:lipschitz}
We assume that for any $h\in [H]$, $Q^\star_h$ is L-Lipschitz continuous: $\text{ for all }  (s,a), (s',a') \in \cal S \times \cal A$
\begin{align*}
    |Q_h^\star(s,a) - Q_h^\star(s',a')| \leq L \cdot \texttt{dist}\left((s,a), (s',a')\right), 
\end{align*}
and without loss of generality,  $$ \texttt{dist}\left((s,a), (s',a')\right) \leq 1, \forall  (s,a), (s',a') \in \cal S \times \cal A.$$
\end{assumption}

Assumption~\ref{assum:lipschitz} tells us that the optimal action values of nearby state-action pairs are close.

For a metric space $\cal X$ and $\epsilon > 0$, we denote the $\epsilon$-net, ${\cal N}(\epsilon) \subset \cal X$, as a set such that
$$ \forall x \in {\cal X}, \quad  \exists x' \in {\cal N}(\epsilon), \quad \texttt{dist}(x, x') \leq \epsilon. $$

If $\cal X$ is compact, we denote $N(\epsilon)$ as the minimum size of an $\epsilon$-net for $\cal X$. The covering dimension $d$ of $\cal X$ is defined 
$$ d \triangleq \inf_{d'}\{d' \geq 0, \forall \epsilon >0 \quad N(\epsilon) \leq \epsilon^{-d'}\}.$$
In particular, if $\cal X$ is a subset of Euclidean space equipped with $l_p$ distance then its covering dimension is at most the linear dimension of $\cal X$. In many applications of interests, state-action spaces are commonly thought to be concentrated near a lower-dimensional manifold lying in high-dimensional ambient space. In this case, the covering dimension is much smaller than the linear dimension of the ambient space.

Covering is closely related to packing. We denote an $\epsilon$-packing, ${\cal M}(\epsilon) \subset \cal X$, as a set such that
$$\forall x, x' \in {\cal M}(\epsilon), \quad \texttt{dist}(x, x') > \epsilon. $$
If $\cal X$ is compact, we denote $M(\epsilon)$ as the maximum size of an $\epsilon$-packing. $N(\epsilon)$ and $M(\epsilon)$ have the same scaling as we have $M(2 \epsilon) \leq N(\epsilon) \leq M(\epsilon)$.

\section{The \textsc{ZoomRL} algorithm} \label{sec:the algorithm}

The \textsc{ZoomRL} algorithm, shown in Algorithm~\ref{alg:main}, incrementally builds an optimistic estimate of the optimal action-value function over $\cal X$. The main idea is to estimate $Q$-values precisely in near-optimal regions, while estimating it loosely in sub-optimal regions. To implement this idea, we learn a partition of the space by zooming in more promising and frequently visited regions. 

\textsc{ZoomRL} maintains a partition of the space $\cal X$ that consists of a growing set of balls, of various sizes.
Initially the set contains a single ball which includes the entire state-action space. Over time the set is expanded to include additional balls.
The algorithm assigns two quantities to each ball: the number of times the ball is selected and an optimistic estimate of the $Q$-value of its center.
By interpolating between these estimates using the Lipschitz structure, the algorithm assigns a tighter upper bound (called \textit{index}) of the  $Q$-value of each ball's center. These \textit{indices} are then used to select the next ball and the next action to execute (line 8 -10 of Algorithm~\ref{alg:main}). Based on the received reward and the observed next state, the algorithm updates the selected ball's statistics (cf line 14-17 of Algorithm~\ref{alg:main}). Then, one ball may be created inside the selected ball according to an \textit{activation rule} that reflects a bias-variance tradeoff (line 19-22 of Algorithm~\ref{alg:main}).

We denote by ${\cal B}_h$ the set of balls at step $h \in [H]$ that may change from episode to episode. Each ball $B = \{ x \in {\cal X}, \dist(x_B, x) \leq \rad(B) \}$ has a radius $\rad(B)$, center $x_B = (s_B, a_B)$ and a domain. The domain of ball $B \in {\cal B}_h$, denoted by $\dom_h(B)$, is defined as the subset of $B$ that excludes all active balls in ${\cal B}_h$ that have radius strictly smaller than $\rad(B)$, i.e 
\begin{equation*}
   \dom_h(B) \triangleq B \setminus \left( \cup_{ \substack{B' \in {\cal B}_h \\ \rad(B') < \rad(B) }} B'\right). 
\end{equation*}
$B$ is called \textit{relevant} to a state $s$ at step $h$ if $(s, a) \in \dom_h(B)$ for some action $a \in \cal A$. We denote the set of relevant balls to a given state $s$ at step $h$ by $\rel_h(s) \triangleq \{B \in {\cal B}_h: \exists a \in {\cal A}, (s, a) \in \dom_h(B)\}$. For each ball $B \in {\cal B}_h$ for some $h$, we keep track of the number of times $B$ is selected at step h (denoted by $n_h(B)$), as well as a high probability upper bound (denoted by $\hQ_h(B)$) for the optimal $Q$-value of the center of $B$ (i.e $Q_h^\star(s_B, a_B)$).

Using the Lipschitz continuity assumption, we have that $L \cdot \rad(B) + \hQ_h(B') + L \cdot \dist(x_B, x_{B'})$ is a valid high probability upper bound on $Q^\star(s_B, a_B)$ for any $B' \in {\cal B}_h$. Consequently, we get a tighter (less overoptimistic) upper bound, denoted by $\texttt{index}_h(B)$, by taking the minimum of these bounds
\begin{align*}
    \ind_h& (B)  \triangleq L \cdot \rad(B)  \\
    & + \min_{ 
\substack{B' \in {\cal B}_h \\ \rad(B') \geq \rad(B)}} \{\hQ_h(B') + L \cdot \dist(B, B')\},
\end{align*}
where, by abuse of notation, we write $\dist(B, B') = \dist(x_B, x_{B'})$).

To facilitate the algorithm's description, we introduce episode-indexed versions of the quantities, as shown in algorithm~\ref{alg:main}. We will use $s_h^{k}, a_h^k$ and $B_h^k$ to represent the state, the action and the ball generated at time step $h$ of the $k$-th episode. Moreover, $\hQ_h^k(B)$ and $n^k_h(B)$ are the statistics associated with each ball $B$ at time step $h$ at the beginning of the $k$-th episode.

The algorithm proceeds as follows. Initially, \textsc{ZoomRL} creates a ball centered at arbitrary state-action pair with radius 1, hence covering the whole space. At step $h$ of the $k$-th episode, a state $s^k_h$ is observed, the algorithm finds the set of relevant balls to $s^k_h$ (i.e $\rel^k_h(s^k_h)$) and picks the ball $B^k_h$ with the largest index (i.e $\ind^k_h$) among the relevant balls. Once the ball is selected, an action $a^k_h$ is chosen randomly among actions $a$ satisfying $(s^k_h, a) \in \dom^k_h(B^k_h)$. Action $a^k_h$ is then executed in the environment, a reward $r^k_h$ is obtained and next state $s^k_{h+1}$ is observed. 

Based on the received reward and next state, the algorithm updates the statistics of the selected ball. The number of visits $n_h^k(B^k_h)$ is incremented by 1: $n_h^{k+1}(B^k_h) = n_h^k(B^k_h) + 1$. Let $t = n_h^{k+1}(B^k_h)$, the $Q$-value estimate is updated as follows
\begin{align*}
    \hQ^{k+1}_h(B^k_h) & \leftarrow (1-\alpha_{t}) \hQ^k_h(B^k_h) 
     +  \alpha_{t} \big( r^k_h + \hV^k_{h+1}(s^k_{h+1}) 
     \\ & \quad+ u_{t} + 2L \cdot \texttt{rad}(B^k_h) \big).
\end{align*}
$\alpha_t \triangleq \frac{H+1}{H+t}$ here is a learning rate and $V^k_{h+1}(x^k_{h+1}) = \min \{H, \max_{B \in \texttt{rel}^k_{h+1}(s^k_{h+1})} \texttt{index}^k_{h+1}(B) \}$ is the estimate of the next state's value. The term $u_{t} + 2L \cdot \texttt{rad}(B^k_h)$ corresponds to an exploration bonus used to bound estimation errors on the value function with high probability. The first term of the bonus is set to $u_t = 4 \sqrt{\frac{H^3 \imath}{t}}$ (we use $\imath \triangleq \log(4HK^2/p)$ for $p \in (0,1)$ to denote the log factor). It corresponds to a Hoeffding-style bonus which reflects the sample uncertainty due to insufficient number of samples. The second term $2L \cdot \texttt{rad}(B^k_h)$ accounts for the maximum possible variation of $Q$-values over the selected ball $B^k_h$.

Contrary to Q-values, the value of next state $\hV^k_h(s^k_{h+1})$ is defined over the entire state space, we don't need to maintain $\hV_h$ but we query it whenever we need. In particular, $\hV^k_h(s^k_{h+1})$ is defined by the largest index among the relevant balls to $s^k_{h+1}$ clipped above by $H$. The clipping here is to keep the value estimate into the range of plausible values while preserving the optimism as $H$ is an upper bound on the true optimal value function. 

Finally, \textsc{ZoomRL} may create a new ball according to the following \textit{activation rule}: If $n^k_h(B^k_h) \geq \frac{1}{\rad(B^k_h)^2}$, a new ball $B'$, centered in $(s^k_h, a^k_h)$ and radius $\rad(B') = \rad(B^k_h)/2$, is created and its $Q$ value is initialized to $H$. The activation criterion (which is equivalent to $\frac{1}{\sqrt{n^k_h(B^k_h)}} \leq \rad(B^k_h)$) reflects a tradeoff between the variance of $Q$-value estimate, due to the number of samples (i.e $\frac{1}{\sqrt{n^k_h(B)}}$) and the bias, corresponding to the radius of the ball. 

\begin{figure}[t]
    \centering
    
    \includegraphics[width=0.23\textwidth]{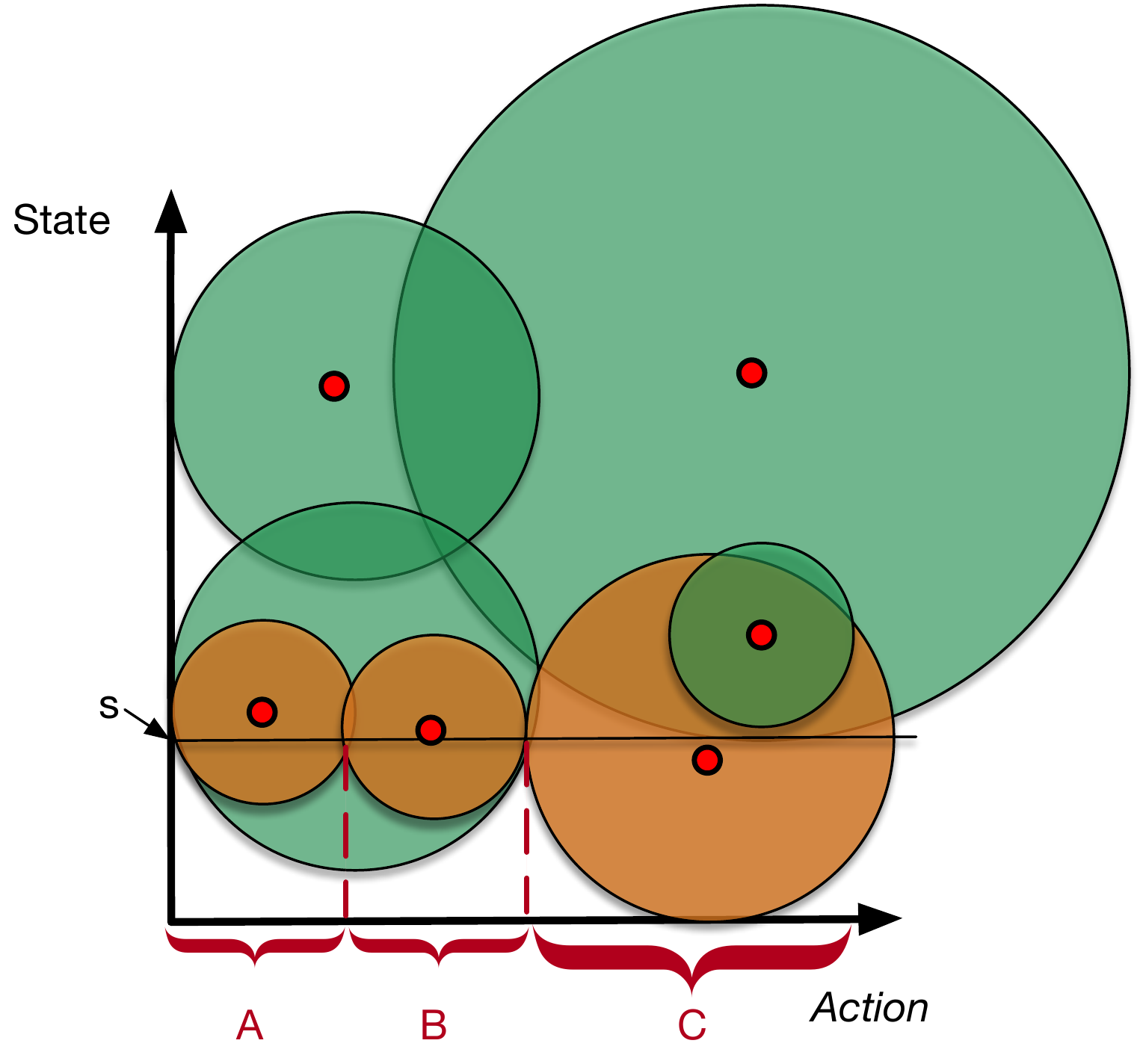}
    \includegraphics[width=0.23\textwidth]{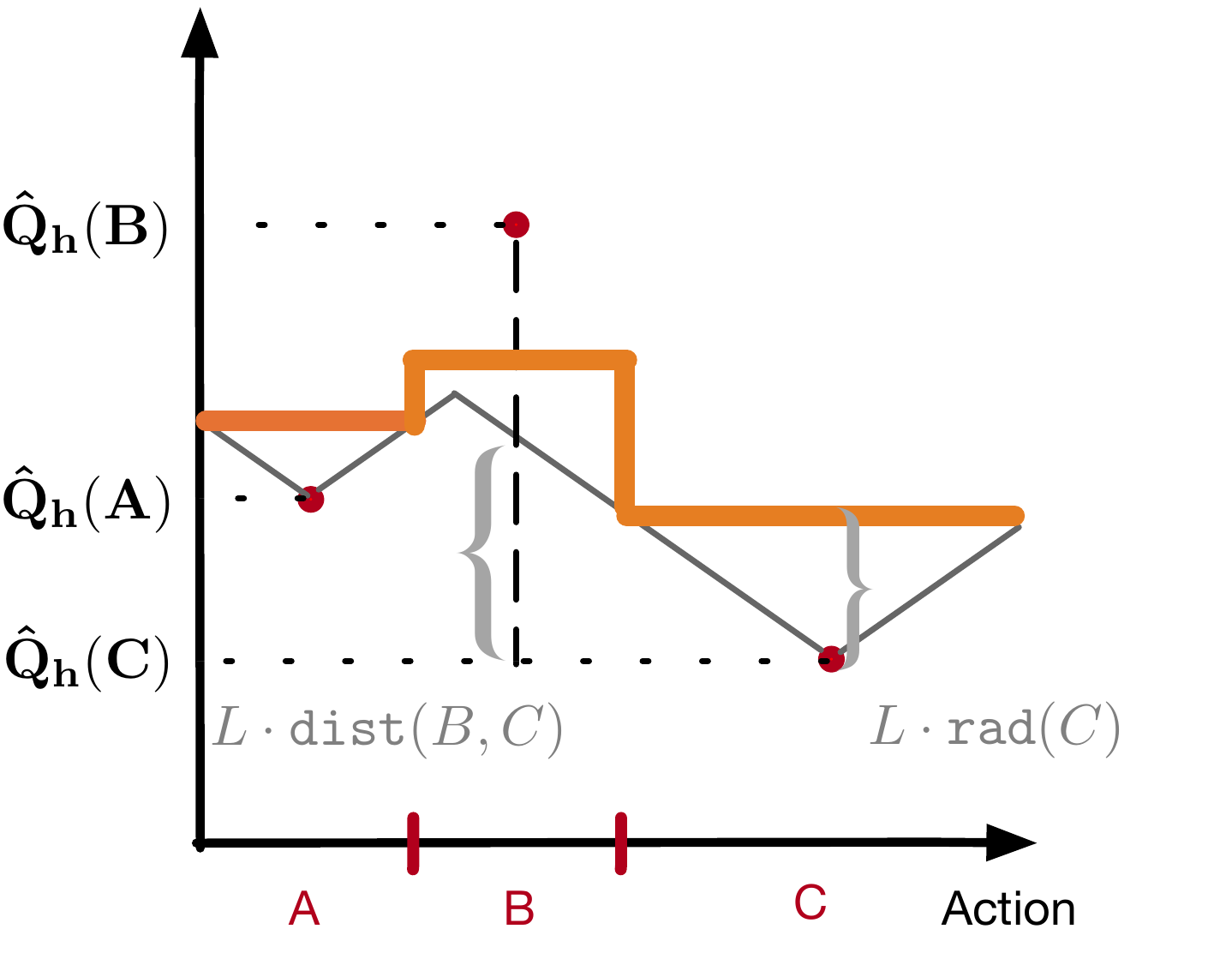}
    \caption{\small Left: a possible partition for $2$-dimensional state-action space. For a given state $s$, we show in \textcolor{orange}{orange} the three relevant balls $A$, $B$ and $C$. Right: For the three relevant balls, we show how the \textcolor{orange}{$\ind$} is constructed based on the interpolation between $Q$-value estimates of each ball. The \textcolor{gray}{gray} piecewise linear curve corresponds to the function $:a \rightarrow \min_{B'} \{\hQ_h(B') + L \cdot \dist( (s, a), x_{B'}) \}$ for a given state.}
    \label{fig:my_label}
\end{figure}

\paragraph{Comparison with~\citet{sinclair2019adaptive}: } Concurrently to our work, ~\citet{sinclair2019adaptive} use a similar approach to learn an adaptive discretization. We highlight here differences between \textsc{ZoomRL} and their algorithm  \textsc{Adaptive Q-learning}:

\begin{compactenum}[\hspace{0pt} 1.]
    \setlength{\itemsep}{2pt}
    \item \textsc{Adaptive Q-learning} requires, at each re-partition step, a packing oracle that is able to take a ball $B$ and value $r$ to output an $r$-packing of $B$. Whereas, our algorithm is oracle-free and creates at most a single child ball when needed.
    \item We leverage more the Lipschitz structure to define the ball's \textit{index}, which is not used in their algorithm.
    \item \citet{sinclair2019adaptive} use an exploration bonus $2\sqrt{\frac{ H^3 \log(4HK/p)}{t}} + \frac{4 L}{\sqrt{t}}$ where $t=n_h^k(B_h^k)$. The first term looks similar to our term $u_t$ with $K^2$ instead of $K$ in the log factor. We think it is due to small issue in their proof because there is a missing union bound over $K$ possible values of the random stopping time $t=n_h^k(B_h^k)$ (cf Proof \ref{proof:high_prob_bound_noise} in appendix). The second term,$\frac{4 L}{\sqrt{t}}$, is different than ours, $L \cdot \rad(B_h^k)$. 
    \item Finally, in \citep{sinclair2019adaptive} each child ball inherents statistics from their parent while in our algorithms the statistics are initialized by zero for $n_h$ and $H$ for $\hQ_h$.
    
\end{compactenum}
\begin{algorithm}[t]
\caption{\textsc{ZoomRL}}\label{alg:main}

\begin{algorithmic}[1]
\STATE {\bf Data:} For $h \in [H]$, we have a  collection ${\cal B}_h $ of  balls.
\STATE {\bf Init:} create ball $B$, with $\texttt{rad}(B) = 1$ and arbitrary center. ${\cal B}^1_h \leftarrow \{ B\}$ for all $h \in [H]$
\STATE $\hQ^1_h(B) = H$ and $n^1_h(B) = 0$, $\forall h \in [H]$
\FOR{ episode $k = 1,\ldots, K$}
	\STATE {\bf Observe} $x^k_1$
	\FOR{step $h=1, \ldots, H$}
	    \STATE {\bf  \textcolor{gray!70!blue}{// Select action}} \label{line:select_action}
	    \STATE $B^k_h \leftarrow \arg\max_{B \in \texttt{rel}^k_h(s^k_h)} \texttt{index}^k_h(B)$
	    \STATE $a^k_h \leftarrow$ any arm $a$ such that $(s^k_h, a) \in \text{dom}(B^k_h)$
	    
	    \STATE Execute action $a^k_h$, observe reward $r^k_h$ and next state $s^k_{h+1}$
	    
	    \STATE {\bf \textcolor{gray!70!blue}{// Query the next value function}} \label{line:update_V}
	    \STATE $ \hV^k_{h+1}(s^k_{h+1}) \leftarrow \min \{H, \max_{B \in \texttt{rel}^k_{h+1}(s^k_{h+1})} \texttt{index}^k_{h+1}(B) \}$
	    
	    \STATE {\bf \textcolor{gray!70!blue}{// Update the selected ball's statistics}}
	    \STATE $ t =  n^{k+1}_h(B^k_h) \leftarrow n^{k}_h(B^k_h) + 1$
	    \STATE $ u_t \leftarrow 4 \sqrt{ \frac{H^3 \imath}{t}}$
	    \STATE \label{line:update_Q} $\hQ^{k+1}_h(B^k_h) \leftarrow (1-\alpha_{t}) \hQ^{k}_h(B^k_h) + \alpha_{t} \left( r^k_h + \hV^k_{h+1}(s^k_{h+1}) + u_t + 2L \cdot \texttt{rad}(B^k_h) \right)$
	    
	    \STATE {\bf \textcolor{gray!70!blue}{// New ball's activation step}}
	    \IF{$n^k_h(B) \geq \frac{1}{\rad(B^k_h)^2}$}
    	    \STATE Create a new ball $B'$ centered in $(s^k_h, a^k_h)$ and radius $\texttt{rad}(B') = \texttt{rad}(B^k_h)/2$
    	    \STATE ${\cal B}^{k+1}_h = {\cal B}^k_h \cup B'$
    	    \STATE $\hQ^{k+1}_h(B') = H$ and $n^{k+1}_h(B') = 0$, $\forall h \in [H]$
	    \ENDIF 
	    
	\ENDFOR
\ENDFOR
\end{algorithmic}
\end{algorithm}
\section{Main results}
In this section, we present our main theoretical result which is an upper bound on the total regret of \textsc{ZoomRL} (see Algorithm~\ref{alg:main}). We start by showing a pessimistic version of the regret bound.
\begin{theorem}[Worst case guarantee]\label{theo: worst_case_bound}
For any $p \in (0, 1)$, with probability $1-p$, the total regret of \textsc{ZoomRL} (see Algorithm~\ref{alg:main}) is at most $O(\sqrt{H^5 \imath} L K^{\frac{d+1}{d+2}})$ where $\imath =  \log(4HK^2/p)$ and $d$ is the covering dimension of the state-action space.
\end{theorem}
The bound in Theorem~\ref{theo: worst_case_bound} matches the regret bound achieved by Net-based Q-learning (\textsc{NbQl}) studied in~\citet{song2019efficient} which assumes access to an $\epsilon$-net of the whole space as input to the algorithm. Moreover, the $\epsilon$-net should be optimal in the sense that the granularity $\epsilon$ of the covering must be chosen in advance ($\epsilon = K^\frac{-1}{d+2}$). Meanwhile, \textsc{ZoomRL} builds the partition on the fly and in data-dependent fashion by allocating more effort in promising regions, which would considerably save the memory requirement in favorable problems while preserving the worst-case guarantee (as shown in Theorem~\ref{theo: worst_case_bound}).

Now, we present a refined regret bound that reflects better the geometry of the underlying space.
\begin{theorem}[Refined regret bound]\label{theo: refined_bound}
For any $p \in (0, 1)$, with probability $1-p$, the total regret of \textsc{ZoomRL} (see Algorithm~\ref{alg:main}) is at most
\begin{align*}
     O\big(&
      (L + \sqrt{H^5 \imath}) \min_{r_0 \in (0,1)} \Big \{ K r_0 + \sum_{ \substack{ r = 2^{-i} \\ r \geq r_0} }\frac{M(r)}{r} \Big \} \\ 
      & + H^2 + \sqrt{H^3 K \imath} \big),
\end{align*}
where $\imath = \log(4H K^2/p)$, $M(r)$ is the $r$-packing number of the state-action space. 
\end{theorem}

Since $M(r)$ is non-increasing in $r$, the leading term (the first term) of the bound of Theorem~\ref{theo: refined_bound} is upper bounded by $\min_{r_0 \in (0,1)}\left \{ K r_0 + \frac{M(r_0)}{r_0} \log(r_0)\right\}$. By setting $r_0 = K^{\frac{-1}{d+2}}$, we recover the worst-case bound in Theorem~\ref{theo: worst_case_bound}. Note that the work of~\citet{sinclair2019adaptive} achieves the same regret bound.

The $r$-packing number $M(r)$ is here to uniformly upper bound the number of balls of radius $r$ generated by the algorithm, as we will see in the analysis deferred to the next section.
Intuitively, balls with small radius would not cover the whole state-action space but rather would be concentrated around near-optimal regions.
We expect that their number would be much smaller that $M(r)$ in practice.

\paragraph{Comparison with contextual bandit setting: }We would like to highlight a negative result of the RL setting comparing to the contextual bandit setting.
When $H=1$ and if we ignore logarithmic factors in Theorem~\ref{theo: refined_bound}, we obtain a bound in $\min_{r_0 \in (0,1)} \left \{ K r_0 + \sum_{ \substack{ r = 2^{-i} \geq r_0} }\frac{M(r)}{r} \right \}$. This looks similar to the regret bound of the contextual Zooming algorithm~\citep{slivkins2014contextual}. But there is a crucial difference: $M(r)$ here is the $r$-packing of the entire space while it is replaced by the $r$-packing of near-optimal regions in~\citet{slivkins2014contextual}. This follows from the fact that in contextual bandit setting, total regret could be straightforwardly written as sum of instant regrets incurred by each ball. Such regret is bounded, up to a multiplicative constant, by the radius of the ball in which the context falls. 

However the dependence of our regret bound on $K$ is still optimal, up to logarithmic factor, with respect to the worst Lipschitz structure. In fact, theorem 8 in~\citet{slivkins2014contextual} states that there exists a distribution $\mathcal{I}$ over problem instances on $(\calS \times \calA, \dist)$ such that for any contextual bandit algorithm, the expected regret over $\mathcal{I}$ is lower bounded by $\Omega\left(\min_{r_0 \in (0,1)} \left \{ K r_0 + \sum_{ \substack{ r = 2^{-i} \geq r_0} }\frac{M(r)}{r} \right \}/ \log(K)\right)$.
\paragraph{Tabular MDP:} In the case of finite state-action MDP without any structural knowledge, one can pick the metric to be $\dist((s, a), (s', a')) = H, \quad \forall (s, a) \neq (s', a')$. It is obvious that the optimal action-value function is 1-Lipschitz (i.e $L=1$) with respect to this metric and that the packing number is at most equal to $|\calS||\calA|$. In this case, if we set $r_0$ to $\sqrt{\frac{|\calS||\calA|}{K}}$, the regret bound in Theorem~\ref{theo: refined_bound} becomes $O(\sqrt{|\calS||\calA| H^5K \imath})$. Hence, we recover exactly the regret bound of $Q$-learning with UCB-Hoeffding algorithm of \citet{jin2018q}. 

\subsection{Result For The Misspecified Case}
We study now how \textsc{ZoomRL} deals with misspecification error. First, we present a formal definition for an approximate Lipscthiz $Q$-value.
\begin{assumption}[Approximately Lipschitz $Q^\star$]
\label{assum:approx_lipschitz}
We assume that for any $h\in [H]$, $Q^\star_h$ can be decomposed as a $L$-Lipschitz continuous term and a bounded term as:
\begin{equation*}
\forall (s, a) \in {\cal X},
    Q^\star_h(s, a) = f_h(s,a) + \Delta_h(s,a),
\end{equation*}
where $\text{ for all }  (s,a), (s',a') \in {\cal X}$
\begin{align*}
    |f_h(s,a) - f_h(s',a')| &\leq L \cdot \dist\left((s,a), (s',a')\right) \text{ and }\\
    |\Delta_h(s, a) | & \leq \epsilon.
\end{align*}
\end{assumption}
A straightforward consequence of Assumption~\ref{assum:lipschitz} is:
\begin{equation*}
    | Q_h^\star(s, a) - Q_h^\star(s',a')| \leq L \cdot \dist((s, a), (s', a')) + 2 \epsilon.
\end{equation*}
The next theorem states that our algorithm, without any modification, is robust to small misspecification error $\epsilon$.
\begin{theorem}[Regret bound in the misspecified case]\label{theo:approx_refined_bound} Suppose that Assumption~\ref{assum:approx_lipschitz} holds.
For any $p \in (0, 1)$, with probability $1-p$, the total regret of \textsc{ZoomRL} (see Algorithm~\ref{alg:main}) is at most
\begin{align*}
     O\big(&
      (L + \sqrt{H^5 \imath}) \min_{r_0 \in (0,1)} \Big \{ K r_0 + \sum_{ \substack{ r = 2^{-i} \\ r \geq r_0} }\frac{M(r)}{r} \Big \} \\ 
      & + H^2 + \sqrt{H^3 K \imath} + \textcolor{red}{HK \epsilon} \big),
\end{align*}
where $\imath = \log(4H K^2/p)$, $M(r)$ is the $r$-packing number of the state-action space. 
\end{theorem}
The Theorem~\ref{theo:approx_refined_bound} states that \textsc{ZoomRL} incurs at most an extra regret term $O(HK \epsilon)$, comparing to Theorem~\ref{theo: refined_bound}. This term is linear in the number of episodes $K$ as well as the error $\epsilon$. The good news is that our algorithm, without any adaptation, does not break down entirely and it enjoys good guarantees when the optimal $Q$-value is close to a Lipschitz function i.e the error $\epsilon$ is small.
\section{Proof Outline}
In this section we outline some key steps in the proof of Theorem ~\ref{theo: refined_bound}. All the omitted proofs as well as the analysis of the misspecified case can be found in the appendix. We start by showing two useful properties of our partitioning scheme.
\begin{lemma}[Partition's properties] \label{lemma:partition_properties}
At each step $h \in [H]$ in episode $k \in [k]$, we have
\begin{enumerate}[label=(\alph*)]
    \item The domains of balls cover the space-action space i.e $\cup_{B \in {\cal B}_h^k} \dom_h^k(B) = \calS \times \calA$.
    \item For any two balls of radius $r >0$, their centers are at distance at least $r$. In other words, for any $r >0$, the set $\{B \in \calB_h^k, \rad(B)=r \}$ forms an $r$-packing for $\calS \times \calA$.
    \end{enumerate}
\end{lemma}
We set $\alpha_t = \frac{H+1}{H+t}$. This specific choice of learning rate comes from~\citet{jin2018q} where they show that this choice is crucial to obtain regret that is not exponential in $H$.
We denote $\alpha_t^0 = \prod_{j=1}^t (1-\alpha_j)$ and $\alpha_t^i = \alpha_i\prod_{j=i+1}^t (1-\alpha_j)$. 
We have $\alpha_t^0 = 0, \forall t \geq 1$ and $\alpha_t^0 = 1$ when $t=0$. The lemma below establishes the recursive formula of $Q$-values estimates for the balls. 
\begin{lemma}\label{lemma:q_formula}

At any $(h, k)\in \times [H]\times [K]$ and $B \in {\cal B}_h^k$, let $t = n_{h}^k(B)$, and suppose $B$ was previously selected at step $h$ of episodes $k_1, k_2, ..., k_{t} < k$. By the update rule of $\hQ$, we have:
\begin{align*}
    \hQ_h^k(B)  = \alpha_{ t }^0 \cdot H + \sum_{i=1}^{ t } \alpha_t^i & \cdot \Big( r_h(x_h^{k_i},a_h^{k_i}) + \hV_{h+1}^{k_i}(x_{h+1}^{k_i}) \\
    & + u_i + 2 L \cdot \rad(B) \Big).
\end{align*}
\end{lemma}
Throughout the learning process, we hope that our estimation $\hQ_h^k$ will get closer to the optimal value $Q_h^\star$, as $k$ increases while we preserve optimism. Using  Azuma-Hoeffding concentration inequality (see Lemma~\ref{lemma:azuma} in appendix) together with the Lipschitz assumption, our next lemma shows that $\hQ_h^k$
is an upper bound on $Q_h^\star$
at any episode $k$ with high probability and the difference between $\hQ_h^k$ and $Q_h^\star$ is controlled by quantities from the next step.
\begin{lemma}\label{lemma:bound_on_Q_k_minus_Q_star}
For any $p \in (0,1)$, we have $\beta_t = 2 \sum_{i=1}^t \alpha_t^i u_i \leq 16 \sqrt{ \frac{H^3 \imath}{t}}$ and, with probability at least $1-p/2$, we have that for all $(s, a,h,k) \in {\cal S} \times {\cal A} \times [H] \times [K]$ and any ball $B \in \calB^k_h$ such that $(s, a) \in \text{dom}_h^k(B)$:
\begin{enumerate}[label=(\alph*)]
    \item $\hQ_h^k(B) \geq Q_h^\star(s, a)$.
    \item $\hQ_h^k(B) - Q_h^\star(x, a) \leq \alpha_t^0 \cdot H + \beta_t +  4 L \cdot \rad(B) +
\sum_{i=1} ^t \alpha_t^i \cdot ( \hV_{h+1}^{k_i} - V_{h+1}^* ) ( s_{h+1}^{k_i} )$.
\end{enumerate}
where $t = n_{h}^k(B)$ and $k_1, \cdots, k_t < k$ are the episodes where $B$ was selected at step $h$.
\end{lemma}
The next lemma translates the optimism in terms of Q-value estimates to optimism in terms of value function estimates. 
\begin{lemma}[Optimism]
\label{lemma:optimism}
Following the same setting as in Lemma~\ref{lemma:bound_on_Q_k_minus_Q_star}, for any $(h, k)$, with probability at least $1-p/2$, we have for any $s \in \cal S$,
$\hV_h^k(s) \geq V_h^\star(s)$.
\end{lemma}
\begin{proof}
Let $s \in \calS$, We have 
$V^\star_{h}(s) = Q^\star_{h}(s, \pi^\star_{h}(s))$. As the set of domains of active balls covers the entire space, there exists $B^\star \in {\cal B}^k_{h+1}$ such that 
$(s, \pi_{h}^\star(s)) \in \texttt{dom}_{h}^k(B^*)$. By the definition of index, we have $\ind^k_{h}(B^\star) = L \cdot \rad(B^\star) + \hQ^k_{h}({\tilde B^\star}) + L \cdot \texttt{dist}({\tilde B^\star}, B^\star)$ for some active ball $\tilde B^\star$. We have 
\begin{align*}
    & \hV^k_{h}(s) - V^\star_{h}(s) \\
    &~ = \min \{H, \max_{B \in \rel_{h}^k(s)} \ind^k_{h}(B) \} - Q^\star_{h}(s, \pi^\star(s))\\
    &~ \geq  \max_{B \in \rel_{h}^k(s)} \ind^k_{h}(B) - Q^\star_{h}(s, \pi^\star(s)) \\
    &~ \geq \ind^k_{h}(B^\star) - Q^\star_{h}(s, \pi^\star(s)) \\
    &~ = L \cdot \rad(B^\star) + \hQ^k_{h}({\tilde B^\star}) + L \cdot \dist({\tilde B^\star}, B^\star) \\
    & \quad - Q^\star_{h}(s, \pi^\star(s)) \\
    &~ \geq L \cdot \rad(B^\star) + Q^\star_{h}(s_{\tilde B^\star}, a_{\tilde B^\star}) + L \cdot \dist({\tilde B^\star}, B^\star) \\
    & ~~ - Q^\star_{h}(s, \pi^\star(s) ) \\
    &~ \geq L \cdot \rad(B^\star) + Q^\star_{h}(s_{ B^\star}, a_{B^\star}) - Q^\star_{h}(s, \pi^\star(s)) \\
    &~ \geq 0
    \end{align*}
Where $(s_{B^\star}, a_{B^\star})$ and $(s_{\tilde B^\star}, a_{\tilde B^\star})$ denote respectively the centers of balls $B^\star$ and $\tilde B^\star$.
The first inequality follows from $Q^\star_{h}(s, a) \leq H $ for any state-action pair $(s, a)$. The third inequality follows from lemma~\ref{lemma:bound_on_Q_k_minus_Q_star}. The fourth and the last inequalities follow from Lipschitz assumption~\ref{assum:lipschitz}
\end{proof}

\subsection{Regret Analysis} \label{sec:regret_analysis}
$\pi_k$ is the policy executed by the algorithm in step $h$ for $H$ steps to reach the end of the episode. By the optimism of our estimates with respect to the true value function (see lemma~\ref{lemma:optimism}), we have with probability at least $1-p/2$ 
\begin{align*}
    \textsc{Regret}(K) & = \sum_{k=1}^K (V^*_1 - V_1^{\pi_k})(x_1^k) \leq \sum_{k=1}^K (\hV^k_1 - V_1^{\pi_k})(x_1^k)
\end{align*}
Denote by $\delta_h^k \triangleq (\hV^k_h -V_h^{\pi_k})(s_h^k)$ and $\phi_h^k \triangleq (\hV^k_h -V^{\star}_h)(s_h^k)$.
As $V^{\star}_h \geq V^{\pi_k}$, we have $\phi_h^k \leq \delta_h^k$. In the sequel, we aim to upper bound $\delta_h^k$ as we have $\textsc{Regret}(K) \leq \sum_{k=1} \delta_1^k$.

Let $B_h^k$ the ball selected at step $h$ of episode $k$ and $B_{\init}$ be the initial ball of radius one that covers the whole space. We denote $B_h^{k, \pa}$ the parent of $B_h^k$. When $B_h^k$ is the initial ball, we consider that it is parent of itself.

\begin{lemma}[Bound on estimation] 
\label{lemma:estimation_bound}
If we denote $\xi_{h+1}^k = [(\P_h - \hat{\P}_h)(V^\star_{h+1} - V^{\pi_k}_h)](s_h^k, a_h^k)$, we have
\begin{align*}
   \delta_h^k & \leq H \alpha_{n_h^k(B_h^k)}^0 \cdot \mathbb{I}_{\{ B_h^k = B_{\init} \}} + (11L + 32 \sqrt{H^3\imath}) \rad(B_h^k) \\
   & 
    + \sum_{i=1}^{n_h^k(B_h^{k,\pa})} \alpha_{n_h^k(B_h^{k, \pa})}^i \phi_{h+1}^{k_i(B_h^{k, \pa})}
    -\phi_{h+1}^k + \delta_{h+1}^k + \xi_{h+1}^k, 
\end{align*}
where $k_i(B_h^{k, \pa})$ is the $i$-th episode where $B_h^{k, \pa}$ is selected at step $h$.
\end{lemma}
Taking the sum over $k \in [K]$ of the estimation bound in lemma~\ref{lemma:estimation_bound},
\begin{align*}
    \sum_{k=1}^K \delta_h^k &
     \leq \msquare 
     + (11L + 32 \sqrt{H^3\imath}) \sum_{k=1}^K \rad(B_h^k) \\ 
    & + \bigtriangleup + \sum_{k=1}^K (-\phi_{h+1}^k + \delta_{h+1}^k + \xi_{h+1}^k),
\end{align*}
where $\msquare = H \sum_{k=1}^K \alpha_{n_h^k(B_h^k)}^0 \cdot \mathbb{I}_{\{ B_h^k = B_{\init} \}}$ and $\bigtriangleup = \sum_{k=1}^K \sum_{i=1}^{n_h^k(B_h^{k,\pa})} \alpha_{n_h^k(B_h^{k, \pa})}^i \phi_{h+1}^{k_i(B_h^{k, \pa})}$. For the fist term, we have $\msquare = H \sum_{k=1}^K \mathbb{I}_{\{ B_h^k = B_{\text{init}}, n_h^k(B_{\text{init}}) = 0 \}} = H$. For the second term $\bigtriangleup$, we regroup the summation in a different way. For every $k' \in [K]$, the term $ \phi_{h+1}^{k'}$ appears in the summation with $k > k'$ when $B_h^{k}$ and $B_h^{k'}$ share the same parent. The first time it appears we have $n_h^k(B_h^{k, \pa}) = n_h^{k'}(B_h^{k', \pa}) + 1$, the second time it appears we have  $n_h^k(B_h^{k, \pa}) = n_h^{k'}(B_h^{k', \pa}) + 2$ and so on. Therefore:
\begin{align*}
     \bigtriangleup
    & \leq \sum_{k'=1}^K \phi_{h+1}^{k'} \sum_{t=n_h^{k'}(B_h^{k', \pa}) + 1 }^{\infty} \alpha_t^{n_h^{k'}(B_h^{k', \pa})} \\
    & \leq \left ( 1 + \frac{1}{H}\right)
    \sum_{k=1}^K \phi_{h+1}^{k}.
\end{align*}
We use in the last inequality $\sum_{t=i }^{\infty} \alpha_t^{i} = 1 + \frac{1}{H}$ (Lemma~\ref{lemma:learning_rate} in appendix). Therefore, using that $\phi_{h+1}^k \leq \delta^k_{h+1}$
\begin{align*}
    & \sum_{k=1}^K \delta_h^k \leq 
    H + (11L + 32 \sqrt{H^3\imath}) \sum_{k=1}^K \rad(B_h^k) \\
    & \quad + \left( 1 + \frac{1}{H}\right)
    \sum_{k=1}^K \phi_{h+1}^{k} + \sum_{k=1}^K(-\phi_{h+1}^k+ \delta_{h+1}^k + \xi_{h+1}^k) \\
    & \leq H + (11L + 32 \sqrt{H^3\imath}) \sum_{k=1}^K \rad(B_h^k) \\
    & \quad + \left( 1 + \frac{1}{H}\right)
    \sum_{k=1}^K \delta_{h+1}^{k} + \sum_{k=1}^K \xi_{h+1}^k.
\end{align*}
By unrolling the last inequality for $h \in [H]$ and using the fact $\delta_{H+1}^k = 0 \quad \forall k \in[K]$, we obtain
\begin{align}
    \sum_{k=1}^K \delta_1^k & \leq 
    \sum_{h=1}^H (1+ \frac{1}{H})^{h-1} \Bigg( H + (11L + 32 \sqrt{H^3\imath})\nonumber \\
    &  \cdot \sum_{k=1}^K \rad(B_h^k)
    +  \sum_{k=1}^K \xi_{h+1}^k \Bigg) \nonumber \\
    & \leq 3 H^2 + 3(11L + 32 \sqrt{H^3\imath}) \sum_{h=1}^H \sum_{k=1}^K \rad(B_h^k) \nonumber \\
    & + 3 \sum_{h=1}^H \sum_{k=1}^K \xi_{h+1}^k. \label{eq:regret0}
\end{align}
The last inequality follows from the fact $\forall h \in [H], \left ( 1 + \frac{1}{H}\right)^{h-1} \leq \left ( 1 + \frac{1}{H}\right)^{H} \leq \exp(1) \leq 3$.

Now, we proceed to upper bound the two terms $\sum_{h=1}^H \sum_{k=1}^K \rad(B_h^k)$ and $\sum_{h=1}^H \sum_{k=1}^K \xi_{h+1}^k$. Using concentration argument, we show with probability at least $1-p/2$, we have (see Lemma~\ref{lemma:bounding xi} in appendix)
\begin{equation} \label{eq:xi_bound}
    \sum_{h=1}^H\sum_{k=1}^{K}  \xi_{h+1}^k \leq 4 \sqrt{2H^3 K \imath} 
\end{equation}

\paragraph{Bounding $\sum_{h=1}^H \sum_{k=1}^K \rad(B_h^k)$: } Let's consider all balls of radius $r$ that have been activated at step $h$ throughout the execution of the algorithm. The maximum number of times a ball $B$ of radius $r$ can be selected before it becomes a parent is upper bounded by $\frac{1}{r^2}$. After ball $B$ becomes a parent, a new ball of radius $r/2$ is created every time $B$ is selected. Therefore, we can write the sum over all ball $B \in {\cal B}^K_h$ of radius $r$ as the sum over set of rounds which consists of the round when $B$ was created and all rounds when $B$ was selected before being a parent.
Let $r_0 \in (0, 1)$. we have
\begin{align}
    \sum_{k=1}^K &\rad(B_h^k)  = \sum_{r = 2^{-i}} \sum_{ \substack{B \in {\cal B}^K_h \\ \rad(B) = r}} \sum_{ \substack{ k \in [K] \\ B_h^k = B}} r \nonumber \\
    & = \sum_{ \substack{ r = 2^{-i} \\ r < r_0} } \sum_{ \substack{B \in {\cal B}^K_h \\ \rad(B) = r}}  \sum_{ \substack{ k \in [K] \\ B_h^k = B}} r +  \sum_{ \substack{ r = 2^{-i} \\  r \geq r_0} } \sum_{ \substack{B \in {\cal B}^K_h \\ \rad(B) = r}}  \sum_{ \substack{ k \in [K] \\ B_h^k = B}} r \nonumber \\
    & \leq K r_0 +  \sum_{ \substack{ r = 2^{-i} \\  r \geq r_0} } \sum_{ \substack{B \in {\cal B}^K_h \\ \rad(B) = r}} \frac{r}{r^2} + 2 r \nonumber \\
    & \leq K r_0 + 3 \sum_{ \substack{ r = 2^{-i} \\  r \geq r_0} } | \{ B \in {\cal B}^K_h: \rad(B) = r\}| \frac{1}{r} \nonumber \\
    & \leq K r_0 + 3 \sum_{ \substack{ r = 2^{-i} \\  r \geq r_0} }\frac{M(r)}{r} \label{eq:radius_bound}
\end{align}
The last step follows from lemma~\ref{lemma:partition_properties}: The set of active balls of radius $r$ is a $r$-packing of ${\cal S} \times {\cal A}$. Thus, $ |\{ B \in {\cal B}^K_h: \rad(B) = r\}| \leq M(r)$ where $M(r)$ is the $r$-covering number.

Plugging bounds \eqref{eq:xi_bound} and \eqref{eq:radius_bound} in \eqref{eq:regret0} and using the union bound, we obtain the desired regret bound in Theorem~\ref{theo: refined_bound}

\section{Conclusion}
In this paper, we present \textsc{ZoomRL}, a provably efficient model-free reinforcement learning algorithm in continuous state-action spaces under the assumption that the true optimal action-value function is Lipscthiz with respect to similarity metric between state-action pairs. Our algorithm takes into account the geometry of the action-value function by allocating more attention to relevant regions. We show that our method achieves sublinear regret that depends on the packing number of the state-action space and that it is robust to small misspecification errors. 

Our method requires the knowledge of the Lipschitz constant $L$ as well as the metric $\dist$ to achieve its performance. A natural future question is whether an RL algorithm can be proved to be efficient without knowing $L$ or $\dist$ in advance.

\newpage
\bibliographystyle{apalike}
\bibliography{lib}

\begin{thebibliography}{}

\bibitem[Asadi et~al., 2018]{asadi2018lipschitz}
Asadi, K., Misra, D., and Littman, M.~L. (2018).
\newblock Lipschitz continuity in model-based reinforcement learning.
\newblock {\em arXiv preprint arXiv:1804.07193}.

\bibitem[Azar et~al., 2014]{azar2014online}
Azar, M.~G., Lazaric, A., and Brunskill, E. (2014).
\newblock Online stochastic optimization under correlated bandit feedback.
\newblock In {\em ICML}, pages 1557--1565.

\bibitem[Azar et~al., 2017]{azar2017minimax}
Azar, M.~G., Osband, I., and Munos, R. (2017).
\newblock Minimax regret bounds for reinforcement learning.
\newblock In {\em Proceedings of the 34th International Conference on Machine
  Learning-Volume 70}, pages 263--272. JMLR. org.

\bibitem[Bubeck et~al., 2009]{bubeck2009online}
Bubeck, S., Stoltz, G., Szepesv{\'a}ri, C., and Munos, R. (2009).
\newblock Online optimization in x-armed bandits.
\newblock In {\em Advances in Neural Information Processing Systems}, pages
  201--208.

\bibitem[Chow and Teicher, 1998]{Chow1998probability}
Chow, Y. and Teicher, H. (1998).
\newblock Probability theory: Independence, interchangeability, martingales.
\newblock {\em Journal of the American Statistical Association}, 93.

\bibitem[Ferns et~al., 2004]{ferns2004metrics}
Ferns, N., Panangaden, P., and Precup, D. (2004).
\newblock Metrics for finite markov decision processes.
\newblock In {\em Proceedings of the 20th conference on Uncertainty in
  artificial intelligence}, pages 162--169. AUAI Press.

\bibitem[Jaksch et~al., 2010]{jaksch2010near}
Jaksch, T., Ortner, R., and Auer, P. (2010).
\newblock Near-optimal regret bounds for reinforcement learning.
\newblock {\em Journal of Machine Learning Research}.

\bibitem[Jian et~al., 2019]{jian2019exploration}
Jian, Q., Fruit, R., Pirotta, M., and Lazaric, A. (2019).
\newblock Exploration bonus for regret minimization in discrete and continuous
  average reward mdps.
\newblock In {\em Advances in Neural Information Processing Systems}, pages
  4891--4900.

\bibitem[Jin et~al., 2018]{jin2018q}
Jin, C., Allen-Zhu, Z., Bubeck, S., and Jordan, M.~I. (2018).
\newblock Is q-learning provably efficient?
\newblock In {\em Advances in Neural Information Processing Systems}, pages
  4863--4873.

\bibitem[Kakade et~al., 2003]{kakade2003exploration}
Kakade, S., Kearns, M.~J., and Langford, J. (2003).
\newblock Exploration in metric state spaces.
\newblock In {\em Proceedings of the 20th International Conference on Machine
  Learning (ICML-03)}, pages 306--312.

\bibitem[Kearns and Singh, 2002]{kearns2002near}
Kearns, M. and Singh, S. (2002).
\newblock Near-optimal reinforcement learning in polynomial time.
\newblock {\em Machine learning}.

\bibitem[Kleinberg et~al., 2008]{kleinberg2008multi}
Kleinberg, R., Slivkins, A., and Upfal, E. (2008).
\newblock Multi-armed bandits in metric spaces.
\newblock In {\em Proceedings of the fortieth annual ACM symposium on Theory of
  computing}, pages 681--690. ACM.

\bibitem[Lakshmanan et~al., 2015]{lakshmanan2015improved}
Lakshmanan, K., Ortner, R., and Ryabko, D. (2015).
\newblock Improved regret bounds for undiscounted continuous reinforcement
  learning.
\newblock In {\em International Conference on Machine Learning}, pages
  524--532.

\bibitem[Lattimore et~al., 2013]{lattimore2013sample}
Lattimore, T., Hutter, M., Sunehag, P., et~al. (2013).
\newblock The sample-complexity of general reinforcement learning.
\newblock In {\em Proceedings of the 30th International Conference on Machine
  Learning}. Journal of Machine Learning Research.

\bibitem[Munos et~al., 2014]{munos2014bandits}
Munos, R. et~al. (2014).
\newblock From bandits to monte-carlo tree search: The optimistic principle
  applied to optimization and planning.
\newblock {\em Foundations and Trends{\textregistered} in Machine Learning},
  7(1):1--129.

\bibitem[Ortner, 2007]{ortner2007pseudometrics}
Ortner, R. (2007).
\newblock Pseudometrics for state aggregation in average reward markov decision
  processes.
\newblock In {\em International Conference on Algorithmic Learning Theory},
  pages 373--387. Springer.

\bibitem[Ortner and Ryabko, 2012]{ortner2012online}
Ortner, R. and Ryabko, D. (2012).
\newblock Online regret bounds for undiscounted continuous reinforcement
  learning.
\newblock In {\em Advances in Neural Information Processing Systems}, pages
  1763--1771.

\bibitem[Pazis and Parr, 2013]{pazis2013pac}
Pazis, J. and Parr, R. (2013).
\newblock Pac optimal exploration in continuous space markov decision
  processes.
\newblock In {\em Twenty-Seventh AAAI Conference on Artificial Intelligence}.

\bibitem[Qian et~al., 2018]{qian2018exploration}
Qian, J., Fruit, R., Pirotta, M., and Lazaric, A. (2018).
\newblock Exploration bonus for regret minimization in undiscounted discrete
  and continuous markov decision processes.
\newblock {\em arXiv preprint arXiv:1812.04363}.

\bibitem[Shreve and Bertsekas, 1978]{shreve1978alternative}
Shreve, S.~E. and Bertsekas, D.~P. (1978).
\newblock Alternative theoretical frameworks for finite horizon discrete-time
  stochastic optimal control.
\newblock {\em SIAM Journal on control and optimization}, 16(6):953--978.

\bibitem[Sinclair et~al., 2019]{sinclair2019adaptive}
Sinclair, S.~R., Banerjee, S., and Yu, C.~L. (2019).
\newblock Adaptive discretization for episodic reinforcement learning in metric
  spaces.
\newblock {\em arXiv preprint arXiv:1910.08151}.

\bibitem[Slivkins, 2014]{slivkins2014contextual}
Slivkins, A. (2014).
\newblock Contextual bandits with similarity information.
\newblock {\em The Journal of Machine Learning Research}, 15(1):2533--2568.

\bibitem[Song and Sun, 2019]{song2019efficient}
Song, Z. and Sun, W. (2019).
\newblock Efficient model-free reinforcement learning in metric spaces.
\newblock {\em arXiv preprint arXiv:1905.00475}.

\bibitem[Strehl and Littman, 2005]{strehl2005theoretical}
Strehl, A.~L. and Littman, M.~L. (2005).
\newblock A theoretical analysis of model-based interval estimation.
\newblock In {\em Proceedings of the 22nd international conference on Machine
  learning}. ACM.

\bibitem[Sutton and Barto, 1998]{sutton1998introduction}
Sutton, R.~S. and Barto, A.~G. (1998).
\newblock {\em Introduction to reinforcement learning}, volume 135.
\newblock MIT Press Cambridge.

\bibitem[Yang et~al., 2019]{yang2019learning}
Yang, L.~F., Ni, C., and Wang, M. (2019).
\newblock Learning to control in metric space with optimal regret.
\newblock {\em arXiv preprint arXiv:1905.01576}.

\bibitem[Zhu and Dunson, 2019]{zhu2019stochastic}
Zhu, X. and Dunson, D. (2019).
\newblock Stochastic lipschitz q-learning.

\end{thebibliography}

\newpage
\appendix
\onecolumn

\section{Outline}
The appendix of this paper is organized as follows:
\begin{compactenum}[\hspace{0pt} 1.]
    \setlength{\itemsep}{2pt}
    \item Appendix \ref{sec:notation} provides a table of notation for easy reference.
    \item Appendix \ref{sec:Omitted proofs for the Lipschitz setting} provides omitted proofs for regret analysis in the Lipschitz setting.
    \item Appendix \ref{sec:misspecified setting} provides the complete regret analysis in the misspecified setting.
    \item Appendix \ref{sec:technical lemmas} provides some technical lemmas.
\end{compactenum}

\section{Notations} \label{sec:notation}
We provide this table for easy reference. Notation will also be defined as it is introduced.

\begin{table*}[h] 
\begin{center}
\caption{Notation table}
\begin{tabular}{l l l}
 \hline
  $L$ & $\triangleq$ & Lipschitz constant of optimal action-value function \\
  $\rad(B)$ & $\triangleq$ & radius of a ball $B$ \\
  $\dist( (s, a), (s', a'))$ & $\triangleq$ & distance between two action-state pairs \\
  $(s_B, a_B)$ & $\triangleq$ & center of a ball $B$ \\
  $\dist(B, B')$ & $\triangleq$ & distance between centers of balls $B$ and $B'$ \\
  $B^{\pa}$ & $\triangleq$ & parent of a ball $B$ \\
  $s_h^k$ &  $\triangleq$ & state encountered in step $h$ of episode $k$ \\
  $a_h^k$ & $\triangleq$ & action taken by the algorithm in step $h$ of episode $k$ \\
  $B_h^k$ & $\triangleq$ & ball selected by the algorithm in step $h$ of episode $k$
  \\
  $B_h^{k, \pa}$ & $\triangleq$ & parent of the selected ball at step $h$ in episode $k$ \\
  $\calB_h^k$ & $\triangleq$ & the set of  balls activated by the algorithm in step $h$ before episode $k$ \\
  $\hQ_h^k(B)$ & $ \triangleq$ & $Q$-value estimate at ball $B$ \\
  $n_h^k(B)$ & $ \triangleq$ & number of times a ball $B$ is selected by the algorithm \\
  $\pi_k$ & $\triangleq$ & policy executed by the algorithm in episode $k$ \\
  $\dom_h^k(B)$ & $\triangleq$ & $ B \setminus \left( \cup_{ \substack{B' \in {\cal B}^k_h \\ \rad(B') < \rad(B) }} B'\right)$, domain of a ball $B$ \\
  $\ind_h^k(B)$ & $\triangleq$ & $ L \cdot \rad(B) + \min_{ 
\substack{B' \in {\cal B}^k_h \\ \rad(B') \geq \rad(B) } } \{\hQ_h(B') + L \cdot \dist(B, B')\}$ \\
$\hV^k_{h}(s)$ & $\triangleq$ &  $\min \{H, \max_{B \in \texttt{rel}^k_{h}(s)} \ind_h^k(B)\}$, value function estimate \\
$p$ & $\triangleq$ & failing probability \\
$\imath$ & $\triangleq$ & $\log(4 HK^2/p)$, log factor\\
$u_t$ & $\triangleq$ & $4 \sqrt{ \frac{H^3 \imath}{t}}$, UCB exploration bonus \\
$\textsc{Regret}(K)$  & $\triangleq$ & $\sum_{k=1}^K (V^*_1 - V_1^{\pi_k})(x_1^k)$ \\
$\delta_h^k$ & $\triangleq$ & $(\hV^k_h -V_h^{\pi_k})(s_h^k)$ \\
$\phi_h^k$ & $\triangleq$ & $(\hV^k_h -V^{\star}_h)(s_h^k)$ \\
$\hat{\P}_h$ & $\triangleq$ & $[\hat{\P}_h V](s_h^k, a_h^k) = V(s^k_{h+1})$ \\
$\xi_{h+1}^k$ & $\triangleq$ & $[(\P_h - \hat{\P}_h)(V^\star_{h+1} - V^{\pi_k}_h)](s_h^k, a_h^k)$ \\
$\alpha_t$ & $\triangleq$ & $\frac{H+1}{H+t}$ \\
$\alpha_t^0$ &  $\triangleq$ & $\prod_{j=1}^t (1-\alpha_j)$ \\ 
$\alpha_t^i$ & $\triangleq$ & $\alpha_i\prod_{j=i+1}^t (1-\alpha_j)$ \\
$M(r)$ & $\triangleq$ & $r$-packing number of the state-action space $\calS \times \calA$ \\
$f_h(s,a)$ & $\triangleq$ & the lipschitz term in the misspecified setting (Assumption~\ref{assum:approx_lipschitz}) \\
$\Delta_h(s, a)$ & $\triangleq$ & $Q_h(s,a) - f_h(s, a)$ \\
$\epsilon$ & $\triangleq$ & misspecification error (uniform bound on $|\Delta_h(s, a)|$)
\end{tabular}

\end{center}
\end{table*}

\newpage
 \section{Omitted proofs for the Lipschitz setting}
 \label{sec:Omitted proofs for the Lipschitz setting}

\subsection{Proof of Lemma~\ref{lemma:partition_properties}}
\begin{enumerate}[label=(\alph*)]
    \item It is obvious that $\cup_{B \in \calB_h^k} \dom_h^k(B) \subset \cup_{B \in \calB_h^k} B$. Let $x \in \cup_{B \in \calB_h^k} B$. Consider a smallest radius ball $B$ in $\calB_h^k$ that contains $x$. Hence, $x \in \dom_h^k(B)$. This shows that $\cup_{B \in \calB_h^k} B \subset \cup_{B \in \calB_h^k} \dom_h^k(B)$ and consequently $\cup_{B \in \calB_h^k} B = \cup_{B \in \calB_h^k} \dom_h^k(B)$. Moreover, $\cup_{B \in \calB_h^k} B = \calS \times \calA$ as it contains the initial ball which covers the whole space.
    \item Let $(B, B') \in \calB_h^k$ two balls of radius $r>0$. Without loss of generality, we suppose that $B$ is created in episode $\tau \leq k$ with parent ball $B^{\pa}$ and $B'$ is created before $\tau$. According to the activation step in \textsc{ZoomRL} algorithm, $(s_h^\tau, a_h^\tau)$ is the center of $B$ and  $(s_h^\tau, a_h^\tau) \in \dom_h^\tau(B^{\pa})$. By the definition of a ball's domain,  $(s_h^\tau, a_h^\tau) \notin B'$, which proves that $\dist(B, B') > r$.
\end{enumerate}

\subsection{Proof of Lemma~\ref{lemma:q_formula}}
\label{app:q_formula}

\begin{proof}
We fix $B \in {\cal B}_h^k$. For notation simplicity, denote $t = n_h^k(B)$. We have:
\begin{align*}
     \hQ^{k}_h(B) 
    = & ~ (1-\alpha_{t}) \cdot \hQ_{h}^{k_t}(B) + \alpha_{t} \cdot \left(r_h(x_h^{k_t},a_h^{k_t}) + \hV_{h+1}^{k_t}(x_{h+1}^{k_t}) + u_t + 2 L \cdot \rad(B) \right) \\
    = & ~ (1-\alpha_{t}) \cdot \left(
    (1-\alpha_{t-1}) \cdot \hQ_h^{k_{t-1}}(B) + \alpha_{t-1}\cdot \left(r_h(x_{h}^{k_{t-1}}, a_{h}^{k_{t-1}}) + \hV_{h+1}^{k_{t-1}}(x_{h+1}^{k_{t-1}}) + u_{t-1} + 2 L \cdot \rad(B) \right) 
    \right)  \\
    & ~ + \alpha_{t} \cdot \left(r_h(x_h^{k_t},a_h^{k_t}) + \hV_{h+1}^{k_t}(x_{h+1}^{k_t}) + u_i + 2 L \cdot \rad(B)\right) \\
    = & ~ \dots\\
    = & ~ \prod_{i=1}^t (1-\alpha_{i}) H + \sum_{i=1}^t \alpha_{i}\prod_{j=i+1}^t (1-\alpha_{j})\left( 
    r_h(x_h^{k_i},a_h^{k_i}) + \hV_{h+1}^{k_{i}}(x_{h+1}^{k_i}) + u_i + 2 L \cdot \rad(B)
    \right) \\
    = & ~ \alpha_{ t }^0 \cdot H + \sum_{i=1}^{ t } \alpha_t^i \cdot \left( r_h(x_h^{k_i},a_h^{k_i}) + \hV_{h+1}^{k_i}(x_{h+1}^{k_i}) +u_i + 2 L \cdot \rad(B) \right),
\end{align*}
where the first step follows from the update rule of $\hQ^k_h$, the second step follows from the update rule for $\hQ_h^{k_{t-1}}$,  the third step follows from recursively representing $\hQ_h^{k_i}$ using $\hQ_h^{k_{i-1}}$ until $i=1$, and the last step follows from the definition of $\alpha_t^0$ and $\alpha_{t}^i$.
\end{proof}

\subsection{Proof of Lemma~\ref{lemma:bound_on_Q_k_minus_Q_star}} 
\begin{proof} Let $B \in {\cal B}_h^k$ and $(s, a) \in \texttt{dom}_h^k(B)$.

Since $\sum_{i=0}^t \alpha_t^i = 1$, we have that $Q^\star_h(s,a) = \alpha_t^0 Q_h^\star(s, a) + \sum_{i=1}^t \alpha_t^i Q_h^*(s, a)$.

By the Lipschitz assumption~\ref{assum:lipschitz} and the fact $\forall i \in [t], (s_h^{k_i}, a_h^{k_i}) \in B$ and $(s, a) \in B$, we have:
\begin{align*}
|Q_h^\star(s_h^{k_i},a_h^{k_i}) - Q_h^\star(s, a)| \leq L \cdot \texttt{dist}((x_h^{k_i},a_h^{k_i}), (x, a)) \leq 2 L \cdot \rad(B).
\end{align*}

Then we have 
\begin{align}
\label{eq:Q_range_lower}
  Q^\star_h(s, a) & \geq \alpha_t^0 Q_h^\star(s,a) + \sum_{i=1}^t \alpha_t^i\left( Q_h^\star(s_h^{k_i},a_h^{k_i}) - 2L \cdot \rad(B) \right) \\
  \label{eq:Q_range_upper}
  Q^\star_h(s, a) & \leq   \alpha_t^0 Q_h^\star(s, a) + \sum_{i=1}^t \alpha_t^i \left(Q_h^\star(s_h^{k_i},a_h^{k_i}) + 2L \cdot \rad(B)\right).
\end{align}

 By Bellman equation, we have $Q_h^\star(s_h^{k_i}, a_h^{k_i}) = r_h(s_h^{k_i},a_h^{k_i}) + [\P_h V_{h+1}^\star](s_h^{k_i},a_h^{k_i})$. Recall $[\hat{\P}_h^{k_i} V_{h+1}](s_h^{k_i},a_h^{k_i} ) = V_{h+1}(s_{h+1}^{k_i})$, we have:
\begin{align*}
    Q_h^\star(s_h^{k_i},a_h^{k_i}) = r_h(s_h^{k_i},a_h^{k_i}) + [(\P_h - \hat{\P}_h^{k_i}) V_{h+1}^\star] (s_h^{k_i}, a_h^{k_i}) + V_{h+1}^\star(s_{h+1}^{k_i}).
\end{align*} 
Substitute the above equality into Eq.~\eqref{eq:Q_range_lower} and ~\eqref{eq:Q_range_upper}, we have:
\begin{align*}
    Q_h^\star(s, a) & \geq \alpha_t^0 Q_h^\star(x, a)) + \sum_{i=1}^t \alpha_t^i \left( r_h(x_h^{k_i},a_h^{k_i}) + [ (\P_h - \hat{\P}_h^{k_i}) V_{h+1}^\star ] (x_h^{k_i}, a_h^{k_i}) + V_{h+1}^\star(x_{h+1}^{k_i}) - 2 L \cdot \rad(B) \right) \\
    Q_h^\star(s, a) & \leq \alpha_t^0 Q_h^\star(x,a)) + \sum_{i=1}^t \alpha_t^i \left( r_h(x_h^{k_i},a_h^{k_i}) + [ (\P_h - \hat{\P}_h^{k_i}) V_{h+1}^\star ] (x_h^{k_i}, a_h^{k_i}) + V_{h+1}^\star(x_{h+1}^{k_i}) + 2 L \cdot \rad(B) \right)
\end{align*}

Subtracting the formula in Lemma~\ref{lemma:q_formula} from the two above inequalities, we have:
\begin{align}
\label{eq:hQ_minus_Q_star_lower}
     \hQ_h^k(B) - Q^\star_h(s, a) & \geq 
      \sum_{i=1}^t \alpha_t^i \left( (\hV_{h+1}^{k_i} - V_{h+1}^\star)(x_{h+1}^{k_i}) + [ (\hat{\P}_h^{k_i} - \P_h)V_{h+1}^* ] (x_h^{k_i},a_h^{k_i}) + u_i
    \right) \\
\label{eq:hQ_minus_Q_star_upper}
    \hQ_h^k(B) - Q^\star_h(s, a) & \leq  \alpha_t^0 H + \sum_{i=1}^t \alpha_t^i \left( (\hV_{h+1}^{k_i} - V_{h+1}^\star)(x_{h+1}^{k_i}) + [ (\hat{\P}_h^{k_i} - \P_h)V_{h+1}^* ] (x_h^{k_i},a_h^{k_i}) + u_i
    + 4 L \cdot \rad(B)
    \right)
\end{align}

\subsubsection{High probability bounds on the sampling noise} \label{proof:high_prob_bound_noise}
To ensure that our estimates concentrate around the true optimal $Q$-values, we need to ensure that the noise terms $[ (\hat{\P}_h^{k_i} - \P_h)V_{h+1}^* ] (x_h^{k_i},a_h^{k_i})$, due to the next states sampling, are not large.

For each ball $B \in {\cal B}_h^k$, $k_i$ is the episode of which $B$ was selected as step $h$ for the $i$-th time. Let ${\cal F}_t$ be the $\sigma$-field generated by all the random variables until episode $t$, step $h$. As $\{ k_i =t\} \in { \cal F}_t$, the random variable $k_i$ is a stopping time. By definition for any $i \geq 0$, $k_i \leq k_{i+1}$ so the $\sigma$-algebra $\calF_{k_i}$ at time $k_i$ satisfies $\calF_{k_i} \subset \calF_{k_{i+1}}$ (see Lemma~\ref{lemma:stopped filtration}). Let's denote  ${\cal G}_i = \calF_{\tau_{i+1}}$ . Then, $( {\cal G}_i)_{i}$ is a filtration. Moreover, via optional stopping \citep{Chow1998probability}, $\left( \mathbb{I}(k_i \leq K) [(\hat{\P}_h^{k_i} - \P) V_{h+1}^*](x_h^{k_i},a_h^{k_i}) \right)_{i=1}^\tau$ is a $2H$-bounded martingale difference sequence w.r.t the filtration $({\cal G}_i)_{i\geq 0}$. By
Azuma-Hoeffding~\ref{lemma:azuma}, we have $\forall t >0, \tau \in [K]$
\begin{equation*}
    \Pr\left[ \left| \sum_{i=1}^{\tau} \alpha_{\tau}^i \cdot \mathbb{I}(k_i \leq K) \cdot [ ( \wh{\P}_h^{k_i} - \P_h ) V_{h+1}^* ] (x_h^{k_i},a_h^{k_i}) \right| \geq t \right] \leq 2 \exp \left( \frac{ -t^2 }{ 8 H^2 \sum_{i=1}^\tau (\alpha^i_\tau)^2 } \right)
\end{equation*}
Let $p \in (0, 1)$, by setting $2 \exp \left( \frac{ -t^2 }{ 8 H^2 \sum_{i=1}^\tau (\alpha^i_\tau)^2 } \right) = \frac{p}{2HK^2}$, we have that for all $\tau \in [k]$ with probability at least $1 - \frac{p}{2HK^2}$: 

\begin{align*}
    \left| \sum_{i=1}^{\tau} \alpha_{\tau}^i \cdot \mathbb{I}(k_i \leq K) \cdot [ ( \wh{\P}_h^{k_i} - \P_h ) V_{h+1}^* ] (x_h^{k_i},a_h^{k_i}) \right| 
    & \leq 2 \sqrt{2} H \sqrt{\sum_{i=1}^\tau (\alpha^i_\tau)^2 \ln(4HK^2/p)}  \\
    & \leq 4 \sqrt{ \frac{H^3 \ln(4H K^2/p)}{\tau}} = 
    4 \sqrt{ \frac{H^3 \imath}{\tau}}
\end{align*}
where the second inequality follows from $\sum_{i=1}^\tau (\alpha^i_\tau)^2 \leq \frac{2 H}{\tau}$ for any $\tau>0$ (see lemma~\ref{lemma:learning_rate}). 
Then by union bound over $\tau \in [K]$, we have with probability at least $1 - \frac{p}{2 HK}$

\begin{equation*}
    \forall \tau \in [K], \quad 
    \left| \sum_{i=1}^{\tau} \alpha_{\tau}^i \cdot \mathbb{I}(k_i \leq K) \cdot [ ( \wh{\P}_h^{k_i} - \P_h ) V_{h+1}^* ] (x_h^{k_i},a_h^{k_i}) \right|  
    \leq 4 \sqrt{ \frac{H^3 \imath}{\tau}}
\end{equation*}
Since the above inequality holds for all $\tau \in [K]$, it also holds for $\tau = t = n_{h}^h(B) \leq K$. We also have that $\mathbf{I}(k_i \leq K)=1$ for any $i \leq n_{h}^h(B)$. As $|{\cal B}_h^k| \leq K$ for all $(h, k) \in [H] \times [K]$, using union bound for all balls and for all steps, we have with probability at least $1 -p/2$: $\forall (h, k) \in [H] \times [K]$ and for all ball $B \in \calB_h^k$,
\begin{equation}\label{eq:ucb0}
    \left| \sum_{i=1}^{t} \alpha_{t}^i \cdot [ ( \wh{\P}_h^{k_i} - \P_h ) V_{h+1}^* ] (x_h^{k_i},a_h^{k_i}) \right|  
    \leq 4 \sqrt{ \frac{H^3 \imath}{t}}, 
     \text{where } t = n_h^k(B)
\end{equation}
According to  Lemma~\ref{lemma:learning_rate}, we have $1/ \sqrt{t} \leq \sum_{i=1} \frac{\alpha_t^i}{t} \leq 2 / \sqrt{t}$. This implies 
$$ 4 \sqrt{ \frac{H^3 \imath}{t}} \leq 4 \sqrt{H^3 \imath} \cdot \sum_{i=1}^t \frac{\alpha_t^i}{t} = \sum_{i=1}^t \alpha_t^i u_i = \beta_t/2 \leq  8 \sqrt{ \frac{H^3 \imath}{t}}$$ 

Then Eq~\eqref{eq:ucb0} gives that,  with probability at least $1 -p/2$: $\forall (h, k) \in [H] \times [K]$ and for all ball $B \in \calB_h^k$,
\begin{equation}\label{eq:ucb1}
    \left| \sum_{i=1}^{t} \alpha_{t}^i \cdot [ ( \wh{\P}_h^{k_i} - \P_h ) V_{h+1}^* ] (x_h^{k_i},a_h^{k_i}) \right|  
    \leq \beta_t/2, 
     \text{where } t = n_h^k(B)
\end{equation}

\subsubsection{Optimism of $Q$-values: Lemma~\ref{lemma:bound_on_Q_k_minus_Q_star} (a)}

We proceed by induction. By definition, we have $\hQ^k_{H+1} = Q^*_{H+1} = 0$ which implies $Q_{H+1}^k(B) - Q_{H+1}^*(s, a) = 0$. Assume that $Q^k_{h+1}(B)-Q^*_{h+1}(s, a) \geq 0$.

Let $i \in [1,t]$, We have 
$V^\star_{h+1}(s^{k_i}_{h+1}) = Q^\star_{h+1}(s^{k_i}_{h+1}, \pi^\star_{h+1}(s^{k_i}_{h+1}))$. As the set of domains of active balls covers the entire space, there exists $B^\star \in {\cal B}^k_{h+1}$ such that 
$(s^{k_i}_{h+1}, \pi_{h+1}^\star(s^{k_i}_{h+1})) \in \texttt{dom}_{h+1}^k(B^*)$. By the definition of index, we have $\texttt{index}^k_{h+1}(B^\star) = L \cdot \rad(B^\star) + \hQ^k_{h+1}({\tilde B^\star}) + L \cdot \texttt{dist}({\tilde B^\star}, B^\star)$ for some active ball $\tilde B^\star$.

We have 
\begin{align*}
    \hV^k_{h+1}(s^{k_i}_{h+1}) - V^\star_{h+1}(s^{k_i}_{h+1})
    & = \min \{H, \max_{B \in \rel_{h+1}^k(s^{k_i}_{h+1})} \ind^k_{h+1}(B) \} - Q^\star_{h+1}(s^{k_i}_{h+1}, \pi^\star(s^{k_i}_{h+1}))\\
    & \geq  \max_{B \in \rel_{h+1}^k(s^{k_i}_{h+1})} \ind^k_{h+1}(B) - Q^\star_{h+1}(s^{k_i}_{h+1}, \pi^\star(s^{k_i}_{h+1})) \\
    & \geq \ind^k_{h+1}(B^\star) - Q^\star_{h+1}(s^{k_i}_{h+1}, \pi^\star(s^{k_i}_{h+1})) \\
    & = L \cdot \rad(B^\star) + \hQ^k_{h+1}({\tilde B^\star}) + L \cdot \dist({\tilde B^\star}, B^\star) - Q^\star_{h+1}(s^{k_i}_{h+1}, \pi^\star(s^{k_i}_{h+1})) \\
    & \geq L \cdot \rad(B^\star) + Q^\star_{h+1}(s_{\tilde B^\star}, a_{\tilde B^\star}) + L \cdot \dist({\tilde B^\star}, B^\star) - Q^\star_{h+1}(s^{k_i}_{h+1}, \pi^\star(s^{k_i}_{h+1}) ) \\
    & \geq L \cdot \rad(B^\star) + Q^\star_{h+1}(s_{ B^\star}, a_{B^\star}) - Q^\star_{h+1}(s^{k_i}_{h+1}, \pi^\star(s^{k_i}_{h+1})) \\
    & \geq 0,
\end{align*}
where $(s_{B^\star}, a_{B^\star})$ and $(s_{\tilde B^\star}, a_{\tilde B^\star})$ denote respectively the centers of balls $B^\star$ and $\tilde B^\star$.
The first inequality follows from $Q^\star_{h+1}(s, a) \leq H $ for any state-action pair $(s, a)$. The third inequality follows from the induction hypothesis. The fourth and the last inequalities follow from Lipschitz assumption~\ref{assum:lipschitz}

Therefore, we have 
\begin{align*}
    Q_h^k(B) - Q_h^*( s, a) &\geq 
    \sum_{i=1}^t \alpha_t^i \cdot \left( ( \hV_{h+1}^{k_i} - V_{h+1}^{\star} ) ( s^{k_i}_{h+1} ) + [ ( \wh{\P}_h^{k_i} - \P_h ) V_{h+1}^\star ] (s_h^{k_i},a_h^{k_i}) + u_i \right) \\
    & \geq - \beta_t/2 + \beta_t/2 = 0.
\end{align*}

\subsubsection{Upper bound: lemma~\ref{lemma:bound_on_Q_k_minus_Q_star} (b) }
We have:
\begin{align*}
    \hQ_h^k(B) - Q_h^\star(s, a) 
    & \leq \alpha_t^0 \cdot H + \sum_{i=1}^t \alpha_t^i \cdot \left( ( V_{h+1}^{k_i} - V_{h+1}^{*} ) ( x^{k_i}_{h+1} ) + [ ( \wh{\P}_h^{k_i} - \P_h ) V_{h+1}^* ] (x_h^{k_i},a_h^{k_i}) + u_i + 4 r(B) \right). \\
& \leq \alpha_t^0 H + \sum_{i=1}^t \alpha_t^i \cdot ( \hV_{h+1}^{k_i} - V_{h+1}^{*} ) ( x^{k_i}_{h+1} ) + \beta_t/2 + \sum_{i=1}^t \alpha_t^i u_i + 4L \sum_{i=1}^t \alpha_t^i \rad(B) \\
& \leq \alpha_t^0 H + \beta_t + 4 L \cdot \rad(B) + \sum_{i=1}^t \alpha_t^i \cdot ( \hV_{h+1}^{k_i} - V_{h+1}^{*} ) ( x^{k_i}_{h+1} ),
\end{align*}
where the second inequality follows from the inequality $\ref{eq:ucb1}$. The third inequality follows from 
$\sum_{i=1}^t \alpha_t^i \leq 1$ 
\end{proof}

\subsection{Proof of lemma~\ref{lemma:estimation_bound}}

Let $B_h^k$ the ball selected at step $h$ of episode $k$ and $B_{\init}$ be the initial ball of radius one that covers the whole space. We need to distinguish between cases where $B_h^k = B_{\init}$ or not.
By the selection step in \textsc{ZoomRL} algorithm, we have $\max_{B \in \rel_h^k(s_h^k)} \ind_h^k(B) = \ind_h^k(B_h^k)$ and $\pi_k(s_h^k) = a_h^k$.

\begin{enumerate}

\item {\bf Case of $B_h^k \neq B_{\init}$: }
We denote $B_h^{k, \pa}$ the parent of $B_h^k$.
\begin{align*}
    \delta_h^k = (\hV^k_h - V_h^{\pi_k})(s_h^k)
    & \leq \max_{B \in \rel_h^k(s_h^k)} \ind_h^k(B) - V_h^{\pi_k}(s_h^k)  \\
    & = \ind_h^k(B_h^k) - Q_h^{\pi_k}(s_h^k, a_h^k)  \\
    & \leq L \cdot \rad(B_h^k) + \hQ_h^k(B_h^{k, \pa}) + L \cdot \dist(B_h^{k, \pa}, B_h^k) - Q_h^{\pi_k}(s_h^k, a_h^k)  \\
    & \leq L \cdot \rad(B_h^k) + \hQ_h^k(B_h^{k, \pa}) + L \cdot \rad(B_h^{k, \pa}) - Q_h^{\pi_k}(s_h^k, a_h^k)  \\
    & = 3 L \cdot \rad(B_h^k) + \underbrace{\hQ_h^k(B_h^{k, \text{par}}) - Q_h^{*}(s_h^k, a_h^k)}_{q_1} + (Q_h^{*} - Q_h^{\pi_k})(s_h^k, a_h^k) 
\end{align*}
The third inequality follows from the fact that the center of $B_h^k$ is in $B_h^{k, \pa}$ and the last equality follows from $\rad(B_h^{k, \pa}) = \rad(B_h^{k})$.

Since $(x_h^k, a_h^k) \in \text{dom}(B_h^k) \subset B_h^{k, \pa}$, we have by Lemma \ref{lemma:bound_on_Q_k_minus_Q_star}
\begin{align*}
    q_1 \leq 
    \alpha_{n_h^k(B_h^{k, \pa})}^0 H + \beta_{n_h^k(B_h^{k, \pa})} +  4 L \cdot \rad(B_h^{k, \pa}) + \sum_{i=1}^{n_h^k(B_h^{k, \text{par}})} \alpha_{n_h^k(B_h^{k, \pa})}^i ( V_{h+1}^{k_i(B_h^{k, \pa})} - V_{h+1}^* ) ( x_{h+1}^{k_i(B_h^{k, \pa})} )
\end{align*}
where we denote by $k_i(B) \in [1, n_h^k(B)]$ the $i$-th episode where $B$ was selected by the algorithm at step $h$. As $B_h^{k, \pa}$ is a parent, we have $n_h^k(B_h^{k, \pa}) > 0$ implying that $\alpha_{n_h^k(B_h^{k, \pa})}^0 = \mathbb{I}\{n_h^k(B_h^{k, \pa}) = 0\} = 0$. Moreover, by the activation rule, we have $\frac{1}{\sqrt{n_h^k(B_h^{k, \pa}) }}\leq \rad(B_h^{k, \pa})$, implying that $\beta_{n_h^k(B_h^{k, \pa})} \leq 16 \sqrt{\frac{H^3 \imath}{n_h^k(B_h^{k, \pa})}} \leq 16 \sqrt{H^3 \imath} \rad(B_h^{k, \pa}) = 32 \sqrt{H^3 \imath} \rad(B_h^k)$. Consequently,

\begin{align*}
    q_1 \leq (8L + 32 \sqrt{H^3\imath}) \rad(B_h^k) + 
    \sum_{i=1}^{n_h^k(B_h^{k,\pa})} \alpha_{n_h^k(B_h^{k, \pa})}^i \phi_{h+1}^{k_i(B_h^{k, \pa})}
\end{align*}
and therefore,
\begin{align}
\label{eq:case1}
    \delta_h^k \leq (11L + 32 \sqrt{H^3\imath}) \rad(B_h^k) + 
    \sum_{i=1}^{n_h^k(B_h^{k,\pa})} \alpha_{n_h^k(B_h^{k, \pa})}^i \phi_{h+1}^{k_i(B_h^{k, \pa})}
    + (Q_h^{*} - Q_h^{\pi_k})(s_h^k, a_h^k)
\end{align}
\item {\bf Case of $B_h^k = B_{\init}$: }

\begin{align}
    \delta_h^k
    & \leq \max_{B \in \rel_h^k(s_h^k)} \ind_h^k(B) - V_h^{\pi_k}(s_h^k)  \nonumber \\
    & = \ind_h^k(B_h^k) - Q_h^{\pi_k}(s_h^k, a_h^k)  \nonumber \\
    & \leq L \cdot \rad(B_h^k) + \hQ_h^k(B_h^{k}) - Q_h^{\pi_k}(s_h^k, a_h^k)  \nonumber \\
    & =  L \cdot \rad(B_h^k) + \hQ_h^k(B_h^{k}) - Q_h^{*}(s_h^k, a_h^k) + (Q_h^{*} - Q_h^{\pi_k})(s_h^k, a_h^k) \nonumber \\
    & \leq 5 L \cdot \rad(B_h^k) + 
    \alpha_{n_h^k(B_h^k)}^0 H + \beta_{n_h^k(B_h^k)} +   \sum_{i=1}^{n_h^k(B_h^k)} \alpha_{n_h^k(B_h^k)}^i \phi_{h+1}^{k_i(B_h^k)} + 
    (Q_h^{*} - Q_h^{\pi_k})(s_h^k, a_h^k) \nonumber \\ 
    & \leq \alpha_{n_h^k(B_h^k)}^0 H + ( 5 L + 16 \sqrt{H^3 \imath} ) \rad(B_h^k) + 
    \sum_{i=1}^{n_h^k(B_h^k)} \alpha_{n_h^k(B_h^k)}^i \phi_{h+1}^{k_i(B_h^k)} + 
    (Q_h^{*} - Q_h^{\pi_k})(s_h^k, a_h^k) \label{eq:case2}
\end{align}
The third inequality follows from lemma~\ref{lemma:bound_on_Q_k_minus_Q_star} and the last inequality follows from the fact that $\rad(B_h^k) = \rad(B_{\init}) = 1$

\end{enumerate}

Now, we can unify the bound \eqref{eq:case1} obtained in the first case where the algorithm selects a ball other than the initial ball and the bound \eqref{eq:case2} in second case where the initial ball is selected. To do that, we consider, by abuse of notation, that the initial ball is parent of itself i.e when $B_h^k=B_{\init}$ we have $B_h^{k, \pa}=B_{\init}$ and we take the maximum over the two bounds

\begin{equation}
    \delta_h^k \leq \alpha_{n_h^k(B_h^k)}^0 H \mathbb{I}_{\{ B_h^k = B_{\init} \}} + (11L + 32 \sqrt{H^3\imath}) \rad(B_h^k) + 
    \sum_{i=1}^{n_h^k(B_h^{k,\pa})} \alpha_{n_h^k(B_h^{k, \pa})}^i \phi_{h+1}^{k_i(B_h^{k, \pa})}
    + (Q_h^{*} - Q_h^{\pi_k})(s_h^k, a_h^k)
\end{equation}
we obtain the desired result but noting that 
\begin{align*}
    (Q_h^{\star} - Q_h^{\pi_k})(s_h^k, a_h^k) & = [\P_h(V_h^{\star} - V_h^{\pi_k})](s_h^k, a_h^k) \\
    & = 
    [(\P_h - \hat{\P}_h )(V_{h+1}^{\star} - V_{h+1}^{\pi_k})](s_h^k, a_h^k)
    + (V_{h+1}^{\star} - V_{h+1}^{\pi_k})(s_{h+1}^k) \\
    & =  [(\P_h - \hat{\P}_h )(V_{h+1}^{\star} - V_{h+1}^{\pi_k})](s_h^k, a_h^k) + (V^\star - \hV^{\pi_k})(s_{h+1}^k) 
    + (\hV_{h+1}^k - V_{h+1}^{\pi_k})(s_{h+1}^k) \\
    & =  \xi_{h+1}^k -\phi_{h+1}^k
    + \delta_{h+1}^k
\end{align*}

\subsection{Bounding $\sum_{h=1}^H \sum_{k=1}^K \xi_{h+1}^k$}
\begin{lemma}\label{lemma:bounding xi}
With probability at least $1-p/2$, we have
\begin{equation*}
    \sum_{h=1}^H\sum_{k=1}^{K}  \xi_{h+1}^k \leq 4 \sqrt{2H^3 K \imath} 
\end{equation*}
\end{lemma}
 Let ${\cal F}_{k, h}$ be the $\sigma$-field generated by all the random variables until episode $k$, step $h$. Then, $\xi_{h+1}^k = [(\P_h - \wh{\P}_h)(V_{h+1}^* - V_{h+1}^{\pi_k})](s_h^k, a_h^k)$ is a martingale difference sequence w.r.t the filtration $\{{\cal F}_{k, h}\}_{k, h\geq 0}$ bounded by $4H$. By
Azuma-Hoeffding (lemma~\ref{lemma:azuma}), we have $\forall t >0, 
    \Pr\left[ \left| \sum_{h=1}^H\sum_{k=1}^{K}  \xi_{h+1}^k \right| \geq t \right] \leq 2 \exp \left( \frac{ -t^2 }{ 32 H^3 K } \right)
$
Therefore,
\begin{align*}
    \Pr\left[ \left| \sum_{h=1}^H\sum_{k=1}^{K}  \xi_{h+1}^k \right| \geq 4 \sqrt{2H^3 K \imath} \right] & \leq 2 \exp \left( \frac{ -32 H^3 K \imath }{ 32 H^3 K } \right) \\
    & = 2 \frac{p}{4 H^2K^2} \leq p/2
\end{align*}
Hence, with probability at least $1-p/2$, we have
\begin{equation*}
    \sum_{h=1}^H\sum_{k=1}^{K}  \xi_{h+1}^k \leq 4 \sqrt{2H^3 K \imath} 
\end{equation*}

\newpage

\section{Misspecified Setting: Approximately Lipschtiz Case} \label{sec:misspecified setting}
The proof structure is similar to the structure in
Appendix B. We will particularly focus on the parts that require different treatments in the misspecified
setting.

\subsection{Recursive Formula of $\hQ_h^k(B) - Q^\star_h(s, a)$}

Let $B \in {\cal B}_h^k$ and $(s, a) \in \dom_h^k(B)$.

Since $\sum_{i=0}^t \alpha_t^i = 1$, we have that $Q^\star_h(s,a) = \alpha_t^0 Q_h^\star(s, a) + \sum_{i=1}^t \alpha_t^i Q_h^*(s, a)$.

By the $\epsilon$-approximately Lipschitz assumption~\ref{assum:approx_lipschitz} and the fact $\forall i \in [t], (s_h^{k_i}, a_h^{k_i}) \in B$ and $(s, a) \in B$, we have:
\begin{align*}
|Q_h^\star(s_h^{k_i},a_h^{k_i}) - Q_h^\star(s, a)| \leq L \cdot \texttt{dist}((x_h^{k_i},a_h^{k_i}), (x, a)) + 2 \epsilon \leq 2 L \cdot \rad(B) + 2\epsilon.
\end{align*}

Then we have 
\begin{align}
\label{eq:approx_Q_range_lower}
  Q^\star_h(s, a) & \geq \alpha_t^0 Q_h^\star(s,a) + \sum_{i=1}^t \alpha_t^i\left( Q_h^\star(s_h^{k_i},a_h^{k_i}) - 2L \cdot \rad(B) -2 \epsilon\right) \\
  \label{eq:approx_Q_range_upper}
  Q^\star_h(s, a) & \leq   \alpha_t^0 Q_h^\star(s, a) + \sum_{i=1}^t \alpha_t^i \left(Q_h^\star(s_h^{k_i},a_h^{k_i}) + 2L \cdot \rad(B) + 2 \epsilon \right).
\end{align}

 By Bellman equation, we have $Q_h^\star(s_h^{k_i}, a_h^{k_i}) = r_h(s_h^{k_i},a_h^{k_i}) + [\P_h V_{h+1}^\star](s_h^{k_i},a_h^{k_i})$. Recall $[\hat{\P}_h^{k_i} V_{h+1}](s_h^{k_i},a_h^{k_i} ) = V_{h+1}(s_{h+1}^{k_i})$, we have:
\begin{align*}
    Q_h^\star(s_h^{k_i},a_h^{k_i}) = r_h(s_h^{k_i},a_h^{k_i}) + [(\P_h - \hat{\P}_h^{k_i}) V_{h+1}^\star] (s_h^{k_i}, a_h^{k_i}) + V_{h+1}^\star(s_{h+1}^{k_i}).
\end{align*} 
Substitute the above equality into Eq.~\eqref{eq:approx_Q_range_lower} and ~\eqref{eq:approx_Q_range_upper}, we have:
\begin{align*}
    Q_h^\star(s, a) & \geq \alpha_t^0 Q_h^\star(x,a)) + \sum_{i=1}^t \alpha_t^i \left( r_h(x_h^{k_i},a_h^{k_i}) + [ (\P_h - \hat{\P}_h^{k_i}) V_{h+1}^\star ] (x_h^{k_i}, a_h^{k_i}) + V_{h+1}^\star(x_{h+1}^{k_i}) - 2 L \cdot \rad(B) - 2 \epsilon\right) \\
    Q_h^\star(s, a) & \leq \alpha_t^0 Q_h^\star(x,a)) + \sum_{i=1}^t \alpha_t^i \left( r_h(x_h^{k_i},a_h^{k_i}) + [ (\P_h - \hat{\P}_h^{k_i}) V_{h+1}^\star ] (x_h^{k_i}, a_h^{k_i}) + V_{h+1}^\star(x_{h+1}^{k_i}) + 2 L \cdot \rad(B) + 2 \epsilon \right)
\end{align*}

Subtracting the formula in Lemma~\ref{lemma:q_formula} from the two above inequalities, we have:
\begin{align}
\label{eq:approx_hQ_minus_Q_star_lower}
     \hQ_h^k(B) - Q^\star_h(s, a) & \geq 
      \sum_{i=1}^t \alpha_t^i \left( (\hV_{h+1}^{k_i} - V_{h+1}^\star)(x_{h+1}^{k_i}) + [ (\hat{\P}_h^{k_i} - \P_h)V_{h+1}^* ] (x_h^{k_i},a_h^{k_i}) + u_i - 2 \epsilon
    \right) \\
\label{eq:approx_hQ_minus_Q_star_upper}
    \hQ_h^k(B) - Q^\star_h(s, a) & \leq  \alpha_t^0 H + \sum_{i=1}^t \alpha_t^i \left( (\hV_{h+1}^{k_i} - V_{h+1}^\star)(x_{h+1}^{k_i}) + [ (\hat{\P}_h^{k_i} - \P_h)V_{h+1}^* ] (x_h^{k_i},a_h^{k_i}) + u_i
    + 4 L \cdot \rad(B) + 2 \epsilon
    \right)
\end{align}

\section{Bounding of $Q_h^k(B) - Q_h^*( s, a)$}

\begin{lemma} ~\label{lemma:approx_bound_on_Q_k_minus_Q_star}
Suppose Assumption \ref{assum:approx_lipschitz} holds. For any $p \in (0,1)$, we have $\beta_t = 2 \sum_{i=1}^t \alpha_t^i u_i \leq 16 \sqrt{ \frac{H^3 \imath}{t}}$ and, with probability at least $1-p/2$, we have that for all $(s, a,h,k) \in {\cal S} \times {\cal A} \times [H] \times [K]$ and any ball $B$ such that $(s, a) \in \text{dom}_h^k(B)$:
\begin{enumerate}[label=(\alph*)]
    
    \item $\hQ_h^k(B) - Q_h^\star(s, a) \geq -4 (H-h+1) \epsilon$
    \item $\hQ_h^k(B) - Q_h^\star(x, a) \leq \alpha_t^0 \cdot H + \beta_t +  4 L \cdot \rad(B) + 2 \epsilon + \sum_{i=1} ^t \alpha_t^i \cdot ( \hV_{h+1}^{k_i} - V_{h+1}^* ) ( s_{h+1}^{k_i} )
$
\end{enumerate}

where $t = n_{h}^k(B)$ and $k_1, \cdots, k_t < k$ are the episodes where $B$ was selected at step $h$.
\end{lemma}
\subsection{High Probability Bound On The Sampling Noise} The same reasoning as in the subsection \ref{proof:high_prob_bound_noise} in the exact lipschitz case gives:
with probability at least $1 -p/2$: $\forall (h, k) \in [H] \times [K]$ and for all ball $B \in \calB_h^k$,
\begin{equation}\label{eq:approx_ucb1}
    \left| \sum_{i=1}^{t} \alpha_{t}^i \cdot [ ( \wh{\P}_h^{k_i} - \P_h ) V_{h+1}^* ] (x_h^{k_i},a_h^{k_i}) \right|  
    \leq \beta_t/2, 
     \text{where } t = n_h^k(B)
\end{equation}

\subsection{Approximate Optimism Of $Q$-values}

We proceed by induction. By definition, we have $\hQ^k_{H+1} = Q^*_{H+1} = 0$ which implies $Q_{H+1}^k(B) - Q_{H+1}^*(s, a) = -4 (H - (H+1) +1) \epsilon $. Assume that $Q^k_{h+1}(B)-Q^*_{h+1}(s, a) \geq -4 (H- (h+1) + 1) \epsilon = -4 (H - h) \epsilon $.

Let $i \in [1,t]$, We have 
$V^\star_{h+1}(s^{k_i}_{h+1}) = Q^\star_{h+1}(s^{k_i}_{h+1}, \pi^\star_{h+1}(s^{k_i}_{h+1}))$. As the set of domains of active balls covers the entire space, there exists $B^\star \in {\cal B}^k_{h+1}$ such that 
$(s^{k_i}_{h+1}, \pi_{h+1}^\star(s^{k_i}_{h+1})) \in \texttt{dom}_{h+1}^k(B^*)$. By the definition of index, we have $\texttt{index}^k_{h+1}(B^\star) = L \cdot \rad(B^\star) + \hQ^k_{h+1}({\tilde B^\star}) + L \cdot \texttt{dist}({\tilde B^\star}, B^\star)$ for some ball $\tilde B^\star$.

We have 
\begin{align*}
    \hV^k_{h+1}(s^{k_i}_{h+1}) - V^\star_{h+1}(s^{k_i}_{h+1})
    & = \min \{H, \max_{B \in \rel_{h+1}^k(s^{k_i}_{h+1})} \ind^k_{h+1}(B) \} - Q^\star_{h+1}(s^{k_i}_{h+1}, \pi^\star(s^{k_i}_{h+1}))\\
    & \geq  \max_{B \in \rel_{h+1}^k(s^{k_i}_{h+1})} \ind^k_{h+1}(B) - Q^\star_{h+1}(s^{k_i}_{h+1}, \pi^\star(s^{k_i}_{h+1})) \\
    & \geq \ind^k_{h+1}(B^\star) - Q^\star_{h+1}(s^{k_i}_{h+1}, \pi^\star(s^{k_i}_{h+1})) \\
    & = L \cdot \rad(B^\star) + \hQ^k_{h+1}({\tilde B^\star}) + L \cdot \dist({\tilde B^\star}, B^\star) - Q^\star_{h+1}(s^{k_i}_{h+1}, \pi^\star(s^{k_i}_{h+1})) \\
    & \geq L \cdot \rad(B^\star) 
    - 4(H -h) \epsilon + 
    Q^\star_{h+1}(s_{\tilde B^\star}, a_{\tilde B^\star})
    + L \cdot \dist({\tilde B^\star}, B^\star) - Q^\star_{h+1}(s^{k_i}_{h+1}, \pi^\star(s^{k_i}_{h+1}) ) \\
    & \geq L \cdot \rad(B^\star) 
    - 4(H - h) \epsilon +
    Q^\star_{h+1}(s_{ B^\star}, a_{B^\star})
    - 2 \epsilon 
    - Q^\star_{h+1}(s^{k_i}_{h+1}, \pi^\star(s^{k_i}_{h+1})) \\
    & \geq -(4(H - h) + 2) \epsilon 
\end{align*}
Where $(s_{B^\star}, a_{B^\star})$ and $(s_{\tilde B^\star}, a_{\tilde B^\star})$ denote respectively the centers of balls $B^\star$ and $\tilde B^\star$.
The first inequality follows from $Q^\star_{h+1}(s, a) \leq H $ for any state-action pair $(s, a)$. The third inequality follows from the induction hypothesis. The fourth and the last inequalities follow from the assumption~\ref{assum:approx_lipschitz}

Therefore, we have 
\begin{align*}
    Q_h^k(B) - Q_h^*( s, a) &\geq 
    \sum_{i=1}^t \alpha_t^i \cdot \left( ( \hV_{h+1}^{k_i} - V_{h+1}^{\star} ) ( s^{k_i}_{h+1} ) + [ ( \wh{\P}_h^{k_i} - \P_h ) V_{h+1}^\star ] (s_h^{k_i},a_h^{k_i}) + u_i - 2 \epsilon \right) \\
    & \geq -(4(H - h) + 2) \epsilon - \beta_t/2 + \beta_t/2 - 2 \epsilon = 
    -4(H - h + 1) \epsilon
\end{align*}

\subsection{Upper Bound of $\hQ_h^k(B) - Q_h^\star(s, a)$}
We have:
\begin{align*}
    \hQ_h^k(B) - Q_h^\star(s, a) 
    & \leq \alpha_t^0 \cdot H + \sum_{i=1}^t \alpha_t^i \cdot \left( ( V_{h+1}^{k_i} - V_{h+1}^{*} ) ( x^{k_i}_{h+1} ) + [ ( \wh{\P}_h^{k_i} - \P_h ) V_{h+1}^* ] (x_h^{k_i},a_h^{k_i}) + u_i + 4 \rad(B) + 2 \epsilon \right). \\
& \leq \alpha_t^0 H + \sum_{i=1}^t \alpha_t^i \cdot ( \hV_{h+1}^{k_i} - V_{h+1}^{*} ) ( x^{k_i}_{h+1} ) + \beta_t/2 + \sum_{i=1}^t \alpha_t^i u_i +\sum_{i=1}^t \alpha_t^i (4 L \cdot \rad(B) + 2 \epsilon) \\
& \leq \alpha_t^0 H + \beta_t + 4 L \cdot \rad(B) + 2 \epsilon + \sum_{i=1}^t \alpha_t^i \cdot ( \hV_{h+1}^{k_i} - V_{h+1}^{*} ) ( x^{k_i}_{h+1} )
\end{align*}
where the second inequality follows from the inequality $\ref{eq:approx_ucb1}$. The third inequality follows from 
$\sum_{i=1}^t \alpha_t^i \leq 1$ 

\begin{lemma}[Approximate Optimism]
\label{lemma:approx_optimism}
Following the same setting as in Lemma~\ref{lemma:approx_bound_on_Q_k_minus_Q_star}, for any $(h, k)$, with probability at least $1-p/2$, we have for any $s \in \cal S$:
\begin{align*}
    \hV_h^k(s) \geq V_h^\star(s) -(4(H-h) + 5)\epsilon
\end{align*}
\end{lemma}
\begin{proof}
Let $s \in \calS$, We have 
$V^\star_{h}(s) = Q^\star_{h}(s, \pi^\star_{h}(s))$. As the set of domains of active balls covers the entire space, there exists $B^\star \in {\cal B}^k_{h+1}$ such that 
$(s, \pi_{h}^\star(s)) \in \texttt{dom}_{h}^k(B^*)$. By the definition of index, we have $\ind^k_{h}(B^\star) = L \cdot \rad(B^\star) + \hQ^k_{h}({\tilde B^\star}) + L \cdot \texttt{dist}({\tilde B^\star}, B^\star)$ for some active ball $\tilde B^\star$.

We have 
\begin{align*}
    \hV^k_{h}(s) - V^\star_{h}(s)
    &~ = \min \{H, \max_{B \in \rel_{h}^k(s)} \ind^k_{h}(B) \} - Q^\star_{h}(s, \pi^\star(s))\\
    &~ \geq  \max_{B \in \rel_{h}^k(s)} \ind^k_{h}(B) - Q^\star_{h}(s, \pi^\star(s)) \\
    &~ \geq \ind^k_{h}(B^\star) - Q^\star_{h}(s, \pi^\star(s)) \\
    &~ = L \cdot \rad(B^\star) + \hQ^k_{h}({\tilde B^\star}) + L \cdot \dist({\tilde B^\star}, B^\star)  - Q^\star_{h}(s, \pi^\star(s)) \\
    &~ \geq L \cdot \rad(B^\star) + Q^\star_{h}(s_{\tilde B^\star}, a_{\tilde B^\star})
    -4 (H -h +1) \epsilon + L \cdot \dist({\tilde B^\star}, B^\star) 
    - Q^\star_{h}(s, \pi^\star(s) ) \\
    &~ \geq L \cdot \rad(B^\star) + Q^\star_{h}(s_{ B^\star}, a_{B^\star}) -2 \epsilon -4 (H -h +1) \epsilon - Q^\star_{h}(s, \pi^\star(s)) \\
    &~ \geq -4 \epsilon -4 (H -h +1) \epsilon = -(4(H-h) + 5)\epsilon
\end{align*}
Where $(s_{B^\star}, a_{B^\star})$ and $(s_{\tilde B^\star}, a_{\tilde B^\star})$ denote respectively the centers of balls $B^\star$ and $\tilde B^\star$.
The first inequality follows from $Q^\star_{h}(s, a) \leq H $ for any state-action pair $(s, a)$. The third inequality follows from lemma~\ref{lemma:approx_bound_on_Q_k_minus_Q_star}. The fourth and the last inequalities follow from assumption~\ref{assum:approx_lipschitz}
\end{proof}

\subsection{Regret Analysis}
By the approximate optimism of our estimates with respect to the true value function (see lemma~\ref{lemma:approx_optimism}), we have with probability at least $1-p/2$ 
\begin{align*}
    \textsc{Regret}(K) & = \sum_{k=1}^K (V^*_1 - V_1^{\pi_k})(x_1^k) \leq \sum_{k=1}^K (\hV^k_1 - V_1^{\pi_k})(x_1^k) + K(4H + 1) \epsilon = \sum_{k=1}^K \delta_1^k + K(4H + 1) \epsilon
\end{align*}

Similarly to the Lemma~\ref{lemma:estimation_bound}, we can show that using Lemma~\ref{lemma:approx_bound_on_Q_k_minus_Q_star} applied on $B_h^{k,\pa}$ the parent of the selected ball at step $h$ of the episode $k$.
\begin{align*}
   \delta_h^k \leq \alpha_{n_h^k(B_h^k)}^0 H \mathbb{I}_{\{ B_h^k = B_{\init} \}} + (11L + 32 \sqrt{H^3\imath}) \rad(B_h^k) +
    \sum_{i=1}^{n_h^k(B_h^{k,\pa})} \alpha_{n_h^k(B_h^{k, \pa})}^i \phi_{h+1}^{k_i(B_h^{k, \pa})}
    -\phi_{h+1}^k + \delta_{h+1}^k + \xi_{h+1}^k 
    + 2\epsilon
\end{align*}
Following the same steps of the Section~\ref{sec:regret_analysis} in the exact lipschitz setting, we obtain
\begin{align*}
    \sum_{k=1}^K \delta_h^k
    \leq H + (11L + 32 \sqrt{H^3\imath}) \sum_{k=1}^K \rad(B_h^k) + \left( 1 + \frac{1}{H}\right)
    \sum_{k=1}^K \delta_{h+1}^{k} + \sum_{k=1}^K \xi_{h+1}^k + 2 K \epsilon
\end{align*}
By unrolling the last inequality for $h \in [H]$ and using the fact $\delta_{H+1}^k = 0 \quad \forall k \in[K]$, we obtain
\begin{align}
    \sum_{k=1}^K \delta_1^k & \leq 
    \sum_{h=1}^H (1+ \frac{1}{H})^{h-1} \Bigg( H + (11L + 32 \sqrt{H^3\imath}) \sum_{k=1}^K \rad(B_h^k)
    +  \sum_{k=1}^K \xi_{h+1}^k + 2 K \epsilon \Bigg) \nonumber \\
    & \leq 3 H^2 + 3(11L + 32 \sqrt{H^3\imath}) \sum_{h=1}^H \sum_{k=1}^K \rad(B_h^k)
    + 3 \sum_{h=1}^H \sum_{k=1}^K \xi_{h+1}^k + 6 HK \epsilon \label{eq:approx_regret0}
\end{align}

Plugging bounds \eqref{eq:xi_bound} and \eqref{eq:radius_bound} from Section~\ref{sec:regret_analysis} in \eqref{eq:approx_regret0} and using union bound, we have with probability $1-p$
\begin{align*}
    \sum_{k=1}^K \delta_1^k \leq O \Bigg( H^2 + \sqrt{H^3K \imath} + (L + \sqrt{H^5 \imath} ) \min_{r_0 \in (0,1)}\Big \{ K r_0 + \sum_{ \substack{ r = 2^{-i} \\ r \geq r_0} }\frac{M(r)}{r} \Big \} + HK \epsilon\Bigg)
\end{align*}
We obtain the regret bound in theorem~\ref{theo:approx_refined_bound} by noting that $\textsc{Regret}(K) \leq \sum_{k=1}^K \delta_1^k + O(HK \epsilon) $.

\newpage

\section{Technical Lemmas} \label{sec:technical lemmas}
\begin{lemma}[Azuma-Hoeffding inequality] \label{lemma:azuma}
Suppose $\{ X_k : k = 0,1,2,3, \cdots \}$ is a martingale and $| X_k - X_{k-1} | < c_k $, almost surely. Then for all positive integers $N$ and all positive reals $t$,
\begin{align*}
    \Pr[ | X_N - X_0 | \geq t ] \leq 2 \exp \left( \frac{ -t^2 }{ 2 \sum_{k=1}^N c_k^2 } \right).
\end{align*}
\end{lemma}

\begin{lemma}[Lemma 4.1 in \citet{jin2018q}] \label{lemma:learning_rate}
The following properties hold for $\alpha^i_t$:
\begin{enumerate}[label=(\alph*)]
    \item $\frac{1}{\sqrt{t}} \leq \sum_{i=1}^t \frac{\alpha^i_t}{\sqrt{i}} \leq \frac{2}{\sqrt{t}}$ for every $t \geq 1$.
    \item $\max_{i \in [t]}\alpha^i_t \leq \frac{2H}{t}$ and $\sum_{i=1}^t (\alpha^i_t)^2 \leq \frac{2H}{t}$ for every $t \geq 1$.
    \item $\sum_{t=i}^\infty \alpha^i_t = 1 + \frac{1}{H}$ for every $i \geq 1$.
\end{enumerate}

\end{lemma}

\subsection{Few Reminders on Probability Theory}
We consider a probability space $(\Omega, {\cal F}, \P)$. We borrow notation from \citet{qian2018exploration}. We call filtration any increasing (for the
inclusion) sequence of sub-$\sigma$-algebras of $\calF$ i.e., $(\calF_n)_{n \in \mathbb{N}}$ where $\forall n \in \mathbb{N}$, $\calF_n \subset \calF_{n+1} \subset \calF$. We denote by $\calF_{\infty} = \cup_{n \in \mathbb{N}} \calF_n$.

\begin{definition}[Stopping time] A random variable $\tau: \Omega: \rightarrow \mathbb{N} \cup \{+ \infty\}$ is called stopping time with respect to a filtration $(\calF_n)_{n \in \mathbb{N}}$ if for all $n \in \mathbb{N}$, $\{ \tau = n\} \in \calF_n$.
\end{definition}

\begin{definition}[$\sigma$-algebra at stopping time] let $\tau$ be a stopping time. An event prior to $\tau$ is any event $A \in \calF_\infty$ s.t $A \cap {\tau =n} \in \calF_n$ for all $n \in \mathbb{N}$. The set of events prior to $\tau$ is a $\sigma$-algebra denoted $\calF_r$ and called $\sigma-$algebra at time $\tau$:
\begin{equation*}
    \calF_r = \{ A \in \calF_\infty, \forall \in \mathbb{N}, A \cap {\tau =n} \in \calF_n \}
\end{equation*}
\end{definition}

\begin{lemma} \label{lemma:stopped filtration}
let $\tau_1$ and $\tau_2$ be two stopping times with respect to the same filtration $(\calF_n)_{n \in \mathbb{N}}$ s.t $\tau_1 \leq \tau_2$ almost surely. Then $\calF_{\tau_1} \subset \calF_{\tau_2}$ 
\end{lemma}

\end{document}